\documentclass[11pt]{article}
\usepackage[dvips]{graphicx}
\usepackage{textcomp}
\usepackage{amsbsy}
\usepackage{latexsym}
\usepackage[mathscr]{eucal}
\usepackage{amsfonts,amsmath,amsthm}
\usepackage{ctable}
\usepackage{amssymb}
\usepackage{float}
\usepackage{authblk}

\usepackage{fullpage}

\newcommand{\AB}[1]{ {\color{black} #1} }

\begin{document}

\title
{Optimal Learning}
\author{ 
  Peter Binev\thanks{
   Department of Mathematics, University of South Carolina, Columbia, SC 29208, binev@math.sc.edu;}, Andrea Bonito\thanks{
   Department of Mathematics,
   Texas A\&M University, College Station, TX 77843, bonito@tamu.edu;}, Ronald DeVore\thanks{
   Department of Mathematics,
   Texas A\&M University, College Station, TX 77843, ronalddevore@tamu.edu;}, and Guergana Petrova\thanks{
   Department of Mathematics,
   Texas A\&M University, College Station, TX 77843, gpetrova@math.tamu.edu.}}

\hbadness=10000
\vbadness=10000
\newtheorem{lemma}{Lemma}[section]
\newtheorem{prop}[lemma]{Proposition}
\newtheorem{cor}[lemma]{Corollary}
\newtheorem{theorem}[lemma]{Theorem}
\newtheorem{remark}[lemma]{Remark}
\newtheorem{example}[lemma]{Example}
\newtheorem{definition}[lemma]{Definition}
\newtheorem{proper}[lemma]{Properties}
\newtheorem{assumption}[lemma]{Assumption}
%
\newenvironment{disarray}{\everymath{\displaystyle\everymath{}}\array}{\endarray}

\def\RR{\rm \hbox{I\kern-.2em\hbox{R}}}
\def\NN{\rm \hbox{I\kern-.2em\hbox{N}}}
\def\ZZ{\rm {{\rm Z}\kern-.28em{\rm Z}}}
\def\CC{\rm \hbox{C\kern -.5em {\raise .32ex \hbox{$\scriptscriptstyle
|$}}\kern
-.22em{\raise .6ex \hbox{$\scriptscriptstyle |$}}\kern .4em}}
\def\vp{\varphi}
\def\<{\langle}
\def\>{\rangle}
\def\t{\tilde}
\def\i{\infty}
\def\e{\varepsilon}
\def\sm{\setminus}
\def\nl{\newline}
\def\o{\overline}
\def\wt{\widetilde}
\def\wh{\widehat}
\def\cT{{\cal T}}
\def\cA{{\cal A}}
\def\cI{{\cal I}}
\def\cV{{\cal V}}
\def\cB{{\cal B}}
\def\cF{{\cal F}}
\def\cY{{\cal Y}}

\def\cD{{\cal D}}
\def\cP{{\cal P}}
\def\cJ{{\cal J}}
\def\cM{{\cal M}}
\def\cO{{\cal O}}
\def\Chi{\raise .3ex
\hbox{\large $\chi$}} \def\vp{\varphi}
\def\lsima{\hbox{\kern -.6em\raisebox{-1ex}{$~\stackrel{\textstyle<}{\sim}~$}}\kern -.4em}
\def\lsim{\hbox{\kern -.2em\raisebox{-1ex}{$~\stackrel{\textstyle<}{\sim}~$}}\kern -.2em}
\def\[{\Bigl [}
\def\]{\Bigr ]}
\def\({\Bigl (}
\def\){\Bigr )}
\def\[{\Bigl [}
\def\]{\Bigr ]}
\def\({\Bigl (}
\def\){\Bigr )}
\def\L{\pounds}
\def\pr{{\rm Prob}}
\newcommand{\cs}[1]{{\color{magenta}{#1}}}
\def\ds{\displaystyle}
\def\ev#1{\vec{#1}}     
\newcommand{\lt}{\ell^{2}(\nabla)}
\def\Supp#1{{\rm supp\,}{#1}}
\def\R{\mathbb{R}}
\def\E{\mathbb{E}}
\def\nl{\newline}
\def\T{{\relax\ifmmode I\!\!\hspace{-1pt}T\else$I\!\!\hspace{-1pt}T$\fi}}
\def\N{\mathbb{N}}
\def\Z{\mathbb{Z}}
\def\N{\mathbb{N}}
\def\Zd{\Z^d}
\def\Q{\mathbb{Q}}
\def\C{\mathbb{C}}
\def\Rd{\R^d}
\def\gsim{\mathrel{\raisebox{-4pt}{$\stackrel{\textstyle>}{\sim}$}}}
\def\sime{\raisebox{0ex}{$~\stackrel{\textstyle\sim}{=}~$}}
\def\lsim{\raisebox{-1ex}{$~\stackrel{\textstyle<}{\sim}~$}}
\def\div{\mbox{ div }}
\def\M{M}  \def\NN{N}                  
\def\L{{\ell}}               
\def\Le{{\ell^1}}            
\def\Lz{{\ell^2}}
\def\Let{{\tilde\ell^1}}     
\def\Lzt{{\tilde\ell^2}}
\def\Ltw{\ell^\tau^w(\nabla)}
\def\t#1{\tilde{#1}}
\def\la{\lambda}
\def\La{\Lambda}
\def\ga{\gamma}
\def\BV{{\rm BV}}
\def\Ga{\eta}
\def\al{\alpha}
\def\cZ{{\cal Z}}
\def\cA{{\cal A}}
\def\cU{{\cal U}}
\def\argmin{\mathop{\rm argmin}}
\def\argmax{\mathop{\rm argmax}}
\def\prob{\mathop{\rm prob}}

\def\cO{{\cal O}}
\def\cA{{\cal A}}
\def\cC{{\cal C}}
\def\cS{{\cal F}}
\def\bu{{\bf u}}
\def\bz{{\bf z}}
\def\bZ{{\bf Z}}
\def\bI{{\bf I}}
\def\cE{{\cal E}}
\def\cD{{\cal D}}
\def\cG{{\cal G}}
\def\cI{{\cal I}}
\def\cJ{{\cal J}}
\def\cM{{\cal M}}
\def\cN{{\cal N}}
\def\cT{{\cal T}}
\def\cU{{\cal U}}
\def\cV{{\cal V}}
\def\cW{{\cal W}}
\def\cL{{\cal L}}
\def\cB{{\cal B}}
\def\cG{{\cal G}}
\def\cK{{\cal K}}
\def\cX{{\cal X}}
\def\cS{{\cal S}}
\def\cP{{\cal P}}
\def\cQ{{\cal Q}}
\def\cR{{\cal R}}
\def\cU{{\cal U}}
\def\bL{{\bf L}}
\def\bl{{\bf l}}
\def\bK{{\bf K}}
\def\bC{{\bf C}}
\def\X{X\in\{L,R\}}
\def\ph{{\varphi}}
\def\D{{\Delta}}
\def\H{{\cal H}}
\def\bM{{\bf M}}
\def\bx{{\bf x}}
\def\bj{{\bf j}}
\def\bG{{\bf G}}
\def\bP{{\bf P}}
\def\bW{{\bf W}}
\def\bT{{\bf T}}
\def\bV{{\bf V}}
\def\bv{{\bf v}}
\def\bt{{\bf t}}
\def\bz{{\bf z}}
\def\bw{{\bf w}}
\def \span{{\rm span}}
\def \meas {{\rm meas}}
\def\rhom{{\rho^m}}
\def\diff{\hbox{\tiny $\Delta$}}
\def\EE{{\rm Exp}}
\def\lll{\langle}
\def\argmin{\mathop{\rm argmin}}
\def\codim{\mathop{\rm codim}}
\def\rank{\mathop{\rm rank}}

\def\argmax{\mathop{\rm argmax}}
\def\dJ{\nabla}
\newcommand{\ba}{{\bf a}}
\newcommand{\bb}{{\bf b}}
\newcommand{\bc}{{\bf c}}
\newcommand{\bd}{{\bf d}}
\newcommand{\bs}{{\bf s}}
\newcommand{\bff}{{\bf f}}
\newcommand{\bp}{{\bf p}}
\newcommand{\bg}{{\bf g}}
\newcommand{\by}{{\bf y}}
\newcommand{\br}{{\bf r}}
\newcommand{\be}{\begin{equation}}
\newcommand{\ee}{\end{equation}}
\newcommand{\bea}{$$ \begin{array}{lll}}
\newcommand{\eea}{\end{array} $$}
\def \Vol{\mathop{\rm  Vol}}
\def \mes{\mathop{\rm mes}}
\def \Prob{\mathop{\rm  Prob}}
\def \exp{\mathop{\rm    exp}}
\def \sign{\mathop{\rm   sign}}
\def \sp{\mathop{\rm   span}}
\def \rad{\mathop{\rm   rad}}
\def \vphi{{\varphi}}
\def \csp{\overline \mathop{\rm   span}}
%
%
\newcommand{\beqn}{\begin{equation}}
\newcommand{\eeqn}{\end{equation}}
\def\beginproof{\noindent{\bf Proof:}~ }
\def\endproof{\hfill\rule{1.5mm}{1.5mm}\\[2mm]}

\newenvironment{Proof}{\noindent{\bf Proof:}\quad}{\endproof}

\renewcommand{\theequation}{\thesection.\arabic{equation}}
\renewcommand{\thefigure}{\thesection.\arabic{figure}}

\makeatletter
\@addtoreset{equation}{section}
\makeatother

\newcommand\abs[1]{\left|#1\right|}
\newcommand\clos{\mathop{\rm clos}\nolimits}
\newcommand\trunc{\mathop{\rm trunc}\nolimits}
\renewcommand\d{d}
\newcommand\dd{d}
\newcommand\diag{\mathop{\rm diag}}
\newcommand\dist{\mathop{\rm dist}}
\newcommand\diam{\mathop{\rm diam}}
\newcommand\cond{\mathop{\rm cond}\nolimits}
\newcommand\eref[1]{{\rm (\ref{#1})}}
\newcommand{\iref}[1]{{\rm (\ref{#1})}}
\newcommand\Hnorm[1]{\norm{#1}_{H^s([0,1])}}
\def\int{\intop\limits}
\renewcommand\labelenumi{(\roman{enumi})}
\newcommand\lnorm[1]{\norm{#1}_{\ell^2(\Z)}}
\newcommand\Lnorm[1]{\norm{#1}_{L_2([0,1])}}
\newcommand\LR{{L_2(\R)}}
\newcommand\LRnorm[1]{\norm{#1}_\LR}
\newcommand\Matrix[2]{\hphantom{#1}_#2#1}
\newcommand\norm[1]{\left\|#1\right\|}
\newcommand\ogauss[1]{\left\lceil#1\right\rceil}
\newcommand{\QED}{\hfill
\raisebox{-2pt}{\rule{5.6pt}{8pt}\rule{4pt}{0pt}}%
  \smallskip\par}
\newcommand\Rscalar[1]{\scalar{#1}_\R}
\newcommand\scalar[1]{\left(#1\right)}
\newcommand\Scalar[1]{\scalar{#1}_{[0,1]}}
\newcommand\Span{\mathop{\rm span}}
\newcommand\supp{\mathop{\rm supp}}
\newcommand\ugauss[1]{\left\lfloor#1\right\rfloor}
\newcommand\with{\, : \,}
\newcommand\Null{{\bf 0}}
\newcommand\bA{{\bf A}}
\newcommand\bB{{\bf B}}
\newcommand\bR{{\bf R}}
\newcommand\bD{{\bf D}}
\newcommand\bE{{\bf E}}
\newcommand\bF{{\bf F}}
\newcommand\bH{{\bf H}}
\newcommand\bU{{\bf U}}
\newcommand\cH{{\cal H}}
\newcommand\sinc{{\rm sinc}}
\def\enorm#1{| \! | \! | #1 | \! | \! |}

\newcommand{\dm}{\frac{d-1}{d}}

\let\bm\bf
\newcommand{\bbeta}{{\mbox{\boldmath$\beta$}}}
\newcommand{\bal}{{\mbox{\boldmath$\alpha$}}}
\newcommand{\bbi}{{\bm i}}

\newcommand{\nnew}[1]{ {\color{black} #1} }
\def\mnew{\color{Blue}}
\def\wnew{\color{magenta}}

\newcommand{\dI}{\Delta}
\newcommand\aconv{\mathop{\rm absconv}}

\maketitle
\date{}
  \begin{abstract}    
  This paper studies  the problem of learning an unknown function $f$ from given data about $f$.  The learning problem is to give an approximation $\hat f$ to
  $f$ that predicts the values of $f$ away from the data.  There are numerous settings for this  learning problem  depending on (i) what additional
  information we have about $f$ (known as a model class assumption), (ii) how we measure the accuracy of how well $\hat f$ predicts $f$, (iii) what is known about the data and data sites, (iv) whether the data observations are polluted by noise. 
  A mathematical description of the optimal performance possible (the smallest possible error of recovery)   is known in the presence of a model class assumption.   Under standard  model class assumptions, it is shown in this paper that a near optimal $\hat f$  can be found by solving a certain
  {finite-dimensional } over-parameterized optimization problem with a penalty term.  Here, near optimal means that the error is bounded by a fixed constant times the optimal error.  This explains the advantage  of over-parameterization which is commonly used in modern machine learning.    The main results of this paper prove  that 
  over-parameterized learning with an appropriate loss function gives a near optimal approximation $\hat f$ of the function $f$ from which the data is collected.  Quantitative bounds are given for how much over-parameterization needs to be employed and how the penalization needs to be scaled in order to guarantee a near optimal recovery of $f$.  An extension of these results to the case where the data is polluted by additive deterministic noise is  also given.\end{abstract}


   
   \section{Introduction}
   \label{S:intro}  
   Learning an unknown function $f$ from  given data observations is a dominant theme in data science.    The central problem is to use the
   data observations of $f$ to construct a function $\hat f$ which approximates $f$ away from the data.  This paper is concerned with    evaluating how well such an  approximation $\hat f$ performs and determining the best possible performance  among all choices of an $\hat f$.  Given answers to these fundamental questions,
   we then turn to the construction of numerical procedures and evaluate their performance against the known best possible performance.
   
   To place ourselves in a firm mathematical setting, we assume that $f$ is in some Banach space $X$ of functions and the performance of
   the approximant $\hat f$ is measured  by $\|f-\hat f\|_X$.  Typical choices for $X$ are the $L_p(\Omega)$ spaces with $\Omega$ a domain in $\R^d$, or  smoothness spaces such as   Sobolev spaces on $\Omega$.  The latter case arises in the context of
  solving  Partial Differential Equations (PDEs).
   
   In the absence of additional information about $f$, it is easy to see that there can be no performance guarantee, i.e. for any choice of $\hat f$,
   the error $\|f-\hat f\|_X$ can be arbitrarily large for a function $f$ which satisfies the data.     The additional information we impose on $f$ is referred to as  model class information. The appropriate model class for a learning problem depends very much on the underlying application and   is a compilation of all that is known about the function $f$ from analysis of the application.  For example, in PDE applications, the model class is typically provided by physics or  regularity theorems for the PDE in hand.  In other applications, such as image or video classification,   appropriate model classes are
   less transparent and open for debate.

    Mathematically, a  {\it model class} 
   is a compact subset $K $ of $X$.  
   Given such a model class, the learning problem is to determine a best approximant  $\hat f$ given only the information that $f$ is in $ K$ and $f$ satisfies
   the given data.  A best function $\hat f$ is called the {\it optimal recovery} of $f$.  
   
   Optimal recovery has the following well-known   mathematical description (see e.g. \cite{MR,TW,DPW}).
 Let us denote the set of all possible candidates for $f$ by $K^*$,  i.e.,
   \be
   \label{K*}
   K^*:=\{f\in K: f\ {\rm satisfies}  \  {\rm the}  \  {\rm data}\}.
   \ee
   { This is a compact subset of $K$.  When we are presented the data, all we know is that it came from some $f\in K$ but we do not know which
   one.  Thus, the task of optimal recovery is to find one function $g\in X$ which simultaneously approximates all elements in  $K^*$ to an error $R$
   with $R$ as small as possible. This best error is described in the next paragraph.
   
     We denote by $B(z,r)_X$ the ball in $X$ of radius $r$ with center $z$.  Given any compact set $S$ in $X$, the  Chebyshev radius $R(S)_X$ of $S$ is defined as
   \be 
   \label{Chebradius}
   R(S)_X:=\inf\{r: \ S\subset B(z,r)_X \ {\rm for \ some \ }  z\in X\}.
   \ee
     While it can happen that $R(S)_X$ may not be attained, the number $R(S)_X$ is well defined and is the {\it error of optimal recovery}. 
     It can also happen that $R(S)_X$ is assumed but the center $z$ is not in $S$.
     
    We now return to our set $S=K^*$. When $R=R(K^*)_X$ is assumed by a ball $B(z,R)$, then $z$ would provide an optimal recovery for $K^*$.  
       For more details on optimal recovery and Chebyshev balls, we refer the reader to 
     \cite{DPW}.}
     
     We return to our set $K^*$.  While the previous paragraph gives a simple mathematical
     description of the optimal recovery of all functions $f\in K^*$, it is nowhere close to giving a numerical { procedure} for learning since  finding an appropriate $z$ is a numerical
     challenge whose difficulty depends on the nature of $K$.  Nevertheless, the radius $R(K^*)_X$ gives a benchmark for measuring the success of a numerical procedure.  If a numerical procedure produces an $\hat f\in X$ that can be shown to give an error
     \be
     \label{nearbest}
     \|f-\hat f\|_X\le CR(K^*),\quad f\in K^*,
     \ee
     it is said to be a {\it near optimal} recovery of $f$ with constant $C$.   Notice that any function $\hat f\in B(K^*)_X$ is a near optimal recovery with constant $2$.
     If a numerical procedure is shown to produce  a near optimal recovery $\hat f$ of $f$, one can rest assured that no other numerical method will perform better
     save for the size of the constant $C$ and issues centering on the numerical cost to implement the { method}.

    \subsection{Dependence on data}
    \label{SS:data}  The error $R(K^*)_X$ of optimal recovery depends very much on the given data.  We assume throughout our paper that this data is given by
     the values of $m$ linear functionals $\lambda_1,\dots,\lambda_m$   applied to $f$.  These linear functionals should be defined for all functions in $K$.  In the simplest setting of noiseless data, the values
     \be
     \label{data}
     w_i:=\lambda_i(f),\quad i=1,\dots,m,
     \ee
     is the data information provided to us about $f$. Instead of $K^*$ we shall use the notation
     \be
   \label{Kw}
   K_w:=\{f\in K: \lambda_i(f)=w_i,\ i=1,\dots,m\}, \quad  w=(w_1, \ldots,w_m),
   \ee
  to indicate the dependence of this set on the data.  With this notation,
  the optimal recovery rate of $f$ from the given information is 
  \be 
  \label{or}
  {\bf optimal \ recovery \  rate}\ = \ R(K_w)_X.
  \ee

  Given that the $m$ data functionals $\lambda_1,\dots,\lambda_m$ are fixed, we
  define
  \be 
  \label{WK}
  W:= W_K:=W(\lambda_1,\dots,\lambda_m;  K):= \{(\lambda_1(g),\dots,\lambda_m(g)):\ g\in K\}\subset \R^m,
  \ee  
  which is the set of all possible data vectors that can arise from observing an
  $f\in K$.   So, $K_w$ is defined for all $w\in W$.  For all other 
  { values $\bar w$, namely for all $\bar w\in\R^m\setminus W$,
  we define $K_{\bar w}:=\emptyset$}.

      The most convenient assumption to make about the $\lambda_j$'s is that they are linear functionals from the dual space $X^*$ of $X$.  However, a common setting in learning is to measure error in the $X=L_2(\Omega,\mu)$ norm, where $\Omega\subset \R^d$  and $\mu$ is a Borel measure, and    to assume that the data are point values of
     $f$, which of course are not linear functionals on all  of $X$ in this case.  The latter case can be treated if the model class $K$ admits point evaluation. A natural assumption in this case is that $K\subset C(\Omega)$, where $C(\Omega)$ is the space of continuous functions defined on $\Omega$.  Another common setting for point evaluation is to assume that $K\subset H$, where $H$ is a reproducing kernel Hilbert space (RKHS) which may be different from the space
     $X$ where we measure performance.    In \S \ref{S:ptvalues},  we study point evaluation in cases where $X$ itself does not admit point evaluations
     as linear functionals.
     
     There are  the following common settings for the data observations:  
     \vskip .1in
     \noindent
     {\bf Setting I:} The common general setting is that  the $\lambda_j$'s are any fixed linear functionals defined on $K$ and we
     had no influence in their choice.  
     \vskip .1in
     \noindent
     {\bf Setting II:} We are free to choose the functionals $\lambda_j$, subject to the restriction that there are only $m$ of them.  
     \vskip .1in
     \noindent
     {\bf Setting III:} The  $\lambda_j$'s   are given by a random selection of $m$ independent draws under some probability distribution.         \vskip .1in
     \noindent
     {\bf Settings IV,V,VI:} The functionals are chosen as in the above cases (I, II, III) but are restricted to come from a dictionary of possible functionals.  Point evaluation falls into this setting.
     \vskip .1in 
     
     Since {\bf Setting I} is the most often used, in this paper we try to stay within this setting as much as possible.  {\bf Setting II} is usually referred to as
     {\it directed learning} and is a well studied setting in functional analysis.  If the functionals are allowed to be any $m$ functionals from $X^*$ with the budget $m$ fixed, then  the best choice for the $m$  functionals gives an optimal recovery rate
     \be
     \label{Gwidth}
     d^m(K)_X:=\inf_{\lambda_1,\dots,\lambda_m\in X^*} \sup_{w\in\R^m} R(K_w)_X
     \ee
     which is known as the {\it Gelfand width} of $K$.
     For classical model classes such as the unit ball $K:=U(Y)$ of a smoothness space $Y$ that embeds into $X$, the Gelfand widths
     are known at least asymptotically as $m\to\infty$.  Standard reference for results on Gelfand widths in classical settings are \cite{P,LGM}
      and the citations therein.  Notice that the Gelfand width would tell us the minimum number  $m$ of measurements needed to guarantee a desired
     accuracy of performance.  Namely,  if we desire recovery error at most $\e$, then we would need $m$ large enough so that $d^m(K)_X\le \e$. 
     
     {\bf Setting III} seems to be the most often studied in modern learning.  The optimal performance in this case is given by the expected width
   \be
   \label{Gave}
   d^m_{\rm ave}(K)_X:= {\rm Exp}_{\lambda_1,\dots,\lambda_m} \sup_{w\in\R^m} R(K_w)_X,
     \ee
   where the expectation is taken with respect to random independent draws according to the underlying probability measure.
   
     If restrictions are placed on which measurement functionals are allowed to be used, then the notions of Gelfand widths and expected widths
     are modified accordingly.  In the case that these functionals are required to be point evaluations, the corresponding Gelfand width is referred to
     as the sampling numbers of $K$ and the information is referred to as {\it standard information} in  the field of Information Based  Complexity (IBC). We will denote sampling numbers by
     \be 
     \label{samplingnumber}
     s_m(K)_X 
     \ee 
     in this paper. Two of the standard references on this line of investigation are \cite{TW,NW}. Because of their importance in learning,  finding the sampling numbers for various model classes $K$ is an active research  topic (see e.g. \cite{CDL, KU, KNS,T}).

     Although this is not the theme of the present paper, let us emphasize that  computing the Gelfand widths and expected widths of model classes $K$ is an important problem in analysis. It is also  important for the learning community since it gives the best performance that would be possible in a numerical {procedure} for learning, and therefore it can serve as a benchmark for evaluating the performance of a particular proposed algorithm.  While quite a bit is known about these widths
     for classical model classes $K$, most of the known results
     are not useful in modern learning.
     Namely, it is known that for classical model classes the sampling numbers (see \cite{KNS}) and Gelfand widths suffer the curse of dimensionality. This precludes the use of these classical model classes in modern learning where the dimension $d$ of the physical space is very large
     (for example $d>10^4$ in many classification problems). 
      Hence, a general open question is to find appropriate model classes in high dimensions that match the  specific  application and then show that their sampling numbers and/or Gelfand widths avoid the curse of dimensionality.

     \subsection{Discretization of the learning problem}
     \label{SS:discretization}
     The above notions are abstract and do not provide a numerical recipe for learning.  Rather, they provide only a benchmark for
     optimal performance.  The goal of learning is to {design a   procedure} that provably converges to an optimal or near optimal recovery of $f$, i.e., reaches the optimal benchmark.  The development of  learning {procedures}  usually proceeds through two stages.  The first is to formulate a
     {finite-dimensional } optimization problem associated to the data.   Here, {finite-dimensional } means that the optimization problem depends on only a finite number of parameters.
     The second stage is to propose and analyze numerical procedures for solving the { finite-dimensional} optimization. 
     In this paper, we shall primarily concern ourselves with the first stage and ask the question: 
     \vskip .1in
     \noindent 
     {\bf Which { finite-dimensional } optimization problems, if they are successfully numerically implemented, are guaranteed to provide the optimal learning possible from the given data?}
     \vskip .1in
     \noindent
     This paper provides an answer to this question in a variety of settings.  Namely, it is shown that the solution to
     a suitable   over-parameterized penalized least squares optimization
     problem gives a near optimal learning {procedure}.  This fact may shed some light on why over-parameterized learning using neural networks is preferred in modern machine learning.
     We touch upon techniques for numerically implementing the { finite-dimensional } optimization   only briefly when we discuss some concrete examples.

     \subsection{Outline of the paper}
     \label{SS:outline}
       In the next section, we begin by recalling the mathematical description of optimal learning {procedures} based on Chebyshev balls.
     The remainder of the  paper concentrates on introducing { finite-dimensional }  minimization formulations, under a model class assumption,  whose solution is near optimal.
     Each of these { finite-dimensional } minimizations  can be taken of the form
     \be 
     \label {hatf}
     \hat f \in   \argmin_{g\in\Sigma_n} \left(\tau \sum_{j=1}^m [w_j-\lambda_j(g)]^2 +\mu \ {\bf pen}_K(g)\right),
     \ee
     where $\tau,\mu>0$ are suitably chosen parameters,  $\Sigma_n$ is a linear space of dimension $n$ or a nonlinear  set described by $n$ parameters, and {\bf pen} is a penalty term depending on the model class $K$.   
     
      We have chosen to call these problems { finite-dimensional } minimization problems since the minimization is performed over the set $\Sigma_n$ depending on a finite number $n$ of parameters.  In going further, we will denote this set simply by $\Sigma$.  The number of its parameters depends on the accuracy $\delta>0$ with which $\Sigma$  approximates $K$. For example,  we will impose that 
     $$
     \dist(K,\Sigma)_X:=\sup_{f\in K}\dist(f,\Sigma)_X=
  \sup_{f\in K}\,\inf_{g\in \Sigma}\|f-g\|_X<\delta,
  $$
  for suitably chosen small enough $\delta>0$. The value of $\delta$   will determine the number $n=n(\delta)$ of parameters needed to describe $\Sigma$. Thus, the smaller the $\delta$, the bigger the  $n$, which corresponds to the fact that the set $\Sigma$ is described by $n$  parameters whose number will be (in general) much larger  than the number $m$ of data available for $f$ (and thus justify the use of the term {\it over-parametrized}
  optimization problem).

     \begin{remark}
     \label{R:unique}
     In general, the minimization problem in \eref{hatf} may not have a unique solution from $\Sigma_n$.  Unless stated otherwise, the statements of the theorems that follow hold for any minimizer.
     \end{remark}
     
     We determine a penalty term for each model class $K$ of $X$ and each of the above settings for the data.  These are given in \S \ref{S:nosect}.
     If the model class $K$ has additional structure, for example, if it is convex and centrally symmetric about the origin, then the penalty can be
      simplified and is presented in \S \ref{SS:convex}. 
      The   case of data consisting of point evaluations  needs a slightly different treatment since the data is no longer necessarily given by a linear functional on $X$.  Point evaluation is considered in \S \ref{S:ptvalues}.

       The above description of the learning problem assumes that the data
        are exact.  A more realistic assumption is that the data observations are corrupted by noise. We have chosen to treat the noiseless case first and then later address how the addition of noise deteriorates the accuracy of best recovery. In \S \ref{S:noisy}, we consider the case when the data observations are corrupted by
     additive deterministic noise.   In this paper, we do not treat the more common assumption in statistics of stochastic noise and the corresponding minimax estimates since 
     the treatment of that case requires substantially new ideas.  However, we do discuss the case of random sampling in \S\ref{S:sampling}. In numerical implementations it is convenient to use other forms of the loss function appearing in \eqref{hatf}. We discuss this aspect in \S\ref{S:variants}.

      In the final section of this paper, we study  a couple of specific 
      settings in learning with the aim of discussing the numerical aspects of implementing the proposed { finite-dimensional } optimization.  In our first example, we treat the case when the model class $K$ is the unit ball of a Sobolev space and the recovery error is measured in $L_2(\Omega)$.  This setting is not realistic for the modern problems of learning, but it does allow us to put forward a specific numerical method for solving the optimal discretization for which convergence of the numerical method is known, namely the  Finite Element Method.
      Our second example is more germane  to modern learning. While we continue to measure error in $L_2(\Omega)$, the model class $K$ is taken as the convex hull of the ReLU single layer neural network dictionary.  We describe  the
      correct optimization problem for an optimal learning {procedure}. While much is known about solutions to this { finite-dimensional }  problem \cite{PN,SX} and numerical methods for solving the optimization problem \cite{Ti,HTT}, very fundamental questions concerning what is the asymptotic
      behavior of the optimal error of recovery are not yet settled.
      This is discussed in more details in \S\ref{SS:Barron}.

    There is a rather vast literature on optimal recovery and learning.  We close this introduction with
   a few remarks which can serve to explain how our results fit into the current literature.  Let  $X$ be the Banach space in whose norm
  $ \|\cdot\|_X$ we measure the optimal recovery or learning error.  If the model class $K$ is the unit ball $K:=\{f:\ \|f\|_Y\le 1\}$ of a Banach space $Y$ which compactly embeds into $X$,  we let  $g$ be the minimum norm interpolant of the data.  That is, $g$ is the function in $Y$
   which satisfies the data and has smallest $Y$ norm.  Obviously, $g$ is in $K$ and is therefore a near optimal recovery with constant $C\le 2$.  If $Y$ is a Hilbert space (regardless of whether $X$ is) then it can be shown that $g$ is actually the Chebyshev center of $K$ 
   and $g$ is an optimal recovery ($C=1$).  The function $g$ is sometimes referred to as an interpolating spline even though it is not necessarily a spline function in the classical sense.  Of course, this does not directly give { an algorithm} since one still must compute $g$.
   There is a general numerical strategy for computing $g$ which rests on computing the Riesz representers of the linear functional $\lambda_j$ viewed as functionals on $Y$.  If one adopts this approach then one still must answer the question of how accurate the computation of the Riesz representers
   must be in order to have near optimality of recovery.   There are some settings in which one can prove that the minimum norm interpolant
   $g$ lives in a finite dimensional linear space or finite dimensional manifold (see \cite{unser2021unifying} and \cite{PN} for recent literature of this type). When such representer theorems  are proven, they provide a powerful numerical tool.  The approach discussed in the present paper differs from the minimum norm interpolant approach in that we treat more general model classes $K$ and we put forward  a { finite-dimensional } optimization problem based on penalties that is always guaranteed to give a near optimal solution.  Moreover, we quantify how fine we must
   choose the discretization $\Sigma_n$ and how small we must choose the penalty parameter $\tau$. 

   Optimization problems of the type proposed here to find near optimal recovery are very common in the literature and fall under the
   notions of Tikhonov regularization, least squares minimization, and LASSO problems { (see \cite{Lbook})}.  While these methods are used often, the typical results in the literature do not describe how this is to be carried out if one wants to guarantee a near optimal recovery. {Certified performance is usually proven in  particular applications (such as sparse signal recovery)  and fixed form of the loss functions (see \cite{BCW,BRT,BLS,MY}). In this paper, we present a general framework  and derive  provable bounds
on  how much over-parameterization is needed and how the penalization
has to be scaled in order to obtain a near optimal recovery of the observed function.}

   \section {Learning in a Banach space setting}
   \label{S:Banach} We begin  by considering  the case where we measure accuracy in a Banach space $X$ and the model class $K$ is simply a compact subset of $X$.  We assume that  the data are the observations
   \eref{data},
   %
   where the $\lambda_j\in X^*$ are linear functionals on $X$.  The vector
   %
  $
   w:=(w_j)_{j=1}^m\in \R^m
   $
   is called the {\it data observation vector} of the unknown $f$ and the $\lambda_j$, $j=1,\dots,m$, are the observation functionals.
    Without loss of generality, we can assume that the $\lambda_j$'s are normalized to have
   norm one, that is, $\|\lambda_j\|_{X^*}=1$, $j=1,\dots, m$. 
 All theorems that follow can be restated in the general case, with the norms of the functionals present as constants at the appropriate places.
   %
   We shall also use the notation
   \be
   \label{lambda}
   \lambda(g):=(\lambda_1(g),\dots,\lambda_m(g))\in\R^m,\quad g\in X.
   \ee
   and \eref{WK}
   %
   throughout this paper. 
 Notice that since $\lambda$ is continuous on $K$ and $K$ is compact in $X$,  the set $W_K$ is a compact subset of $\R^m$.
   Obviously, all these quantities depend  on the $\lambda$ but we generally do not indicate this dependence since we think of the observation  functionals as fixed.
   
   As noted in the introduction, the totality of  information we have about the unknown function $f$ is that $f\in K_w$ for the given data observations $w$. As with $K^*$, we define $B(K_w)_X$ to be a Chebyshev (i.e., smallest) ball in $X$ which contains $K_w$.  An optimal recovery of $f$ is the 
   Chebyshev center $z_w$ of $B(K_w)_X$ and the error of optimal recovery is the Chebyshev radius $R(K_w):=R(K_w)_X$ of $B(K_w)_X$.   The goal of learning  is to find a numerical procedure which would take the data
   and the knowledge of $K$ and create an approximant $\hat f\in X$ such that
   \be
   \label{no1}
   \|f-\hat f\|_X\le C R(K_w)_X,
   \ee
   with a reasonable constant $C$.  We call such an approximant $\hat f$  a {\it near optimal recovery}  for $K_w$ with constant $C$. 
   \begin{remark}
   \label{R:nearoptimal}
   It can happen that $R(K_w)_X$ is zero.  This would mean that there is only one function in $K_w$, i.e., only one function from $K$ that fits the data. In this  case we  would only have near optimality  in the above sense if $\hat f=f$.  To avoid this exceptional case, we   assume
   in going forward that $R(K_w)_X>0$ in the theorems that follow.  It is easy to formulate a version of each of these theorems to handle the case $R(K_w)_X=0$ but we leave that task to the reader.
   \end{remark}
   
    In this paper, we are interested in formulating { finite-dimensional } optimization problems whose solution would provide a near optimal approximant
    $\hat f$ to $f$.   We begin by giving sufficient conditions on a function $\hat f$ to be a near optimal approximant.
    
    \subsection{A preliminary result}
    \label{SS:preliminary}
      It seems very doubtful that a numerical method would find an element $g\in K_w$ when given just $w$ and the knowledge of $K$.  A more reasonable numerical task would be to find a $g\in X$ such that $g$ almost satisfies the data and is close to $K$.  We can formulate the concept of almost satisfying the data in many equivalent ways  since the data observations are finite.  To be concrete, we shall use
      the weighted empirical $\ell_2$ norm
      \be
      \label{l2}
     \|v\|:= \|v\|_{\ell_2(\R^m)}:= \Big[\frac{1}{m}\sum_{j=1}^m |v_j|^2\Big]^{1/2},\quad v\in\R^m.
      \ee
    { Notice that the data mapping $\lambda$ is a linear mapping from $X$ to $\R^m$ whose norm is one when we use \eref{l2} as the norm for $\R^m$.  This means that $\lambda$ is a Lipschitz mapping:
     \be 
     \label{lambdaLip}
     \|\lambda(f)-\lambda(g)\|\le \|f-g\|_X,\quad f,g \in X.
     \ee
     We shall use this fact repeatedly, usually without further mention, in this paper.}
       
  Let us suppose that  when given an $\e>0$ we can  find a $g_\e\in X$ for which 
  \be
  \label{morereasonable}
 \|\lambda(g_\e)-w\|\le \e\quad {\rm and} \quad \dist(g_\e,K)_X\le \e.
  \ee
  A numerical scheme may have the ability to drive $\e$ to zero at the expense of higher levels of computation.
  The question arises as to {\it which level of accuracy  $\e$ would  guarantee 
 that  $g_\e$ provides a near optimal recovery.  Equivalently, we would need $g_\e$ to provide good  approximation to the Chebyshev center $z_w$ of $K_w$.}   To formulate such a bound, we introduce
  the following {\it expanded Chebyshev radius}
  \be
  \label{Rwe}
   R(K(w,\e))_X,\quad \hbox{where}\quad  K(w,\e):=\bigcup_{w':\,\|w'-w\|\le \e} K_{w'}, 
  \ee
  which is the Chebyshev radius of the inflated set  $K(w,\e)$.  Notice that $K_w\subset K(w,\e)$ for all $\e>0$.  We discuss properties of $R(K(w,\e))_X$ in more detail in the next subsection. The behavior of this expanded radius is important for deciding how much over-parameterization is needed for near optimal recovery.  For now, we prove the following lemma.

  \begin{lemma}
  \label{L:Rlimit}  For any compact subset $K$ of $X$ and any $w\in\R^m$, we have
  \be
  \label{LR1}
  \lim_{\e\to 0^+} R(K(w,\e))_X=R(K_w)_X.
  \ee
  \end{lemma}
  \begin{proof}
   Since the sets $K(w,\e)$, $\e>0$,  are nested, the function $R(K(w,\e))_X$, $\e>0$, is decreasing as $\e$ decreases.
  Hence, the limit in \eref{LR1} exists.  Suppose that this limit is $R_0$ and $R_0>R(K_w)_X$.   Let us fix $\xi>0$ such that $R_0>R_0-\xi>R(K_w)_X$. Since for each $\e>0$ we have $R(K(w,\e))_X\geq R_0$,   there is an   $f_\e\in  K(w,\e)$ with $\|f_\e-z_w\|_X\geq  R_0-\xi$, where $z_w$ is the Chebyshev center of $B(K_w)_X$.   
  { Then we take $\varepsilon_n\rightarrow 0$ and consider the corresponding sequence $f_{\varepsilon_n}\in K(w,\varepsilon_n)$ with $\|f_{\e_n}-z_w\|_X\geq  R_0-\xi$. Since $K$ is compact, this sequence has a subsequence 
which converges to
  a limit $f^*$ in $K$ and } $\|f^*-z_w\|_X\ge  R_0-\xi>R(K_w)$. We also know that  $ \lambda(f^*)=w$ and so $f^*\in K_w$.  This means that  
  $\|f^*-z_w\|_X\le R(K_w)$.  This contradicts { the assumption} $R_0>R(K_w)$
  and proves  \eref{LR1}.
  \end{proof}
  
\begin{remark}
\label{R:genconvergence}
Let us record for further use that for any fixed $\gamma > 0$ and $w\in\R^m$, we have
\be
\label{limit1}
\lim_{\e\to 0^+} R(K(w,\gamma+\e))_X = R(K( w,\gamma))_X.
\ee
This is proved as in Lemma \ref{L:Rlimit} by using the fact that the collection of sets
$K(w,\gamma+\e)$, $\e>0$, is a monotone family.
  
\end{remark}

  The following theorem gives a quantitative bound on the recovery performance of a constructed function $g_\e$ in terms of how closely it fits the data and how close it is to the model class $K$.
   \begin{theorem}
  \label{P:approxcenter}
  If $g_\e$ is any function in $X$ satisfying \eref{morereasonable}, then
  \be
  \label{P:a}
 \|f-g_\e \|_X\le \e +2R(K(w,2\e))_X,\quad f\in K_w.
  \ee
  If $R(K_w)\neq 0$, then for any $C>2$ and for  $\e$   suitably small the function  $g_\e$ is a near best recovery of $f$ with constant $C$.
  \end{theorem}
  \begin{proof}
  { We know that there is an $h\in K$ such that}
  $\|g_\e-h\|_X\le \e$.  This $h$ satisfies 
   
  $$
  \|\lambda(h)-w\|\leq
   \|\lambda(h)-\lambda(g_\e)\|+\|\lambda(g_\e)-w\|\leq
  2\e,
  $$
{where we have used \eref{lambdaLip}.} Thus, $h\in K_{w'}$ where $\|w-w'\|\le 2\e$.  Hence,   $h$
  is in the set $K(w,2\e)$. Now any $f\in K_w$ is also in $K(w,2\e)$.  This means that $\|f-h \|_X\le 2R(K(w,2\e))_X$.   Since $\|g_\e-h\|_X\le \e$, we obtain \eref{P:a}.  If $C>2$ and $\e>0$ is sufficiently small, then
  $\e+2 R(K(w,2\e))_X$ is smaller than $CR(K_w)_X$ because of Lemma \ref{L:Rlimit}.
  \end{proof}

  \subsection{The behavior of $R(K_w)_X$}
  The analysis that follows in this paper depends on  the function $\e \mapsto R(K(w,\e))_X$ and so it may be useful to the reader to make
  a few remarks on this function.  Its behavior depends very much on $K$ and needs to be analyzed for each $K$ individually. 
   From the above estimates, we see that a  critical issue in quantitative bounds for the performance of learning { procedures} is the rate of convergence  of $R(K(w,\e))_X$ to $R(K_w)_X$ as $\e\to 0^+$.    It is easy to give examples  of compact sets $K$  
  for which $R(K(w,\e))_X$ tends to $R(K_w)_X$  arbitrarily slowly.  
   Concerning the behavior of $R(K_w)_X$, let us note that this may not be a continuous function of $w$.  To illustrate these issues, we consider the following simple example of a compact set in $\R^2$.
  \vskip .1in
  \noindent
  {\bf Example:} {\it We define  the compact set $K:=[0,1]^2 \cup \left( [1,2] \times \{\frac 12\}\right)\subset X=\mathbb R^2$, equipped with the Euclidean norm.  We take the measurement functional $\lambda$  to be the  first coordinate of a point
  $\bx=(x_1,x_2)\in\R^2$: $\lambda(\bx)=x_1$. 
  {
  Thus, we have
  $$
  K_w=\begin{cases}
\{(w,y)\subset \R^2:\,\,y\in [0,1]\}, \quad \quad w\in [0,1],\\
\{(w,\frac{1}{2})\}, \quad\quad\quad\quad\quad\quad\quad\quad\quad w\in [1,2],\\
\emptyset, \quad \quad\quad \quad\quad\quad\quad\quad\quad\quad\quad\quad w\in (-\infty,1)\cup(2, \infty).
\end{cases}
  $$
  }
  \vskip .1in

 \begin{figure}[ht!]
\begin{center}
\begin{tabular}{cc}
\includegraphics[width=0.4\textwidth]{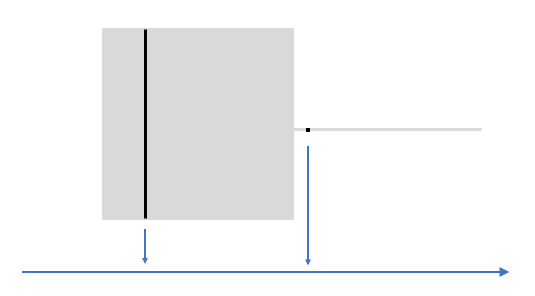}
&
\includegraphics[width=0.4\textwidth]{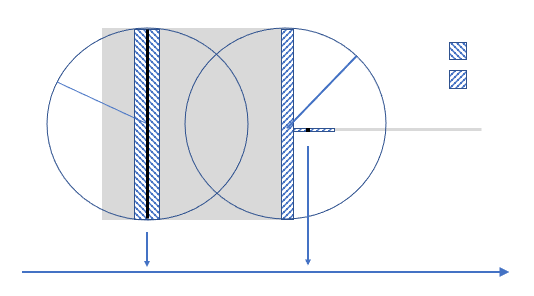} \\
\includegraphics[width=0.3\textwidth]{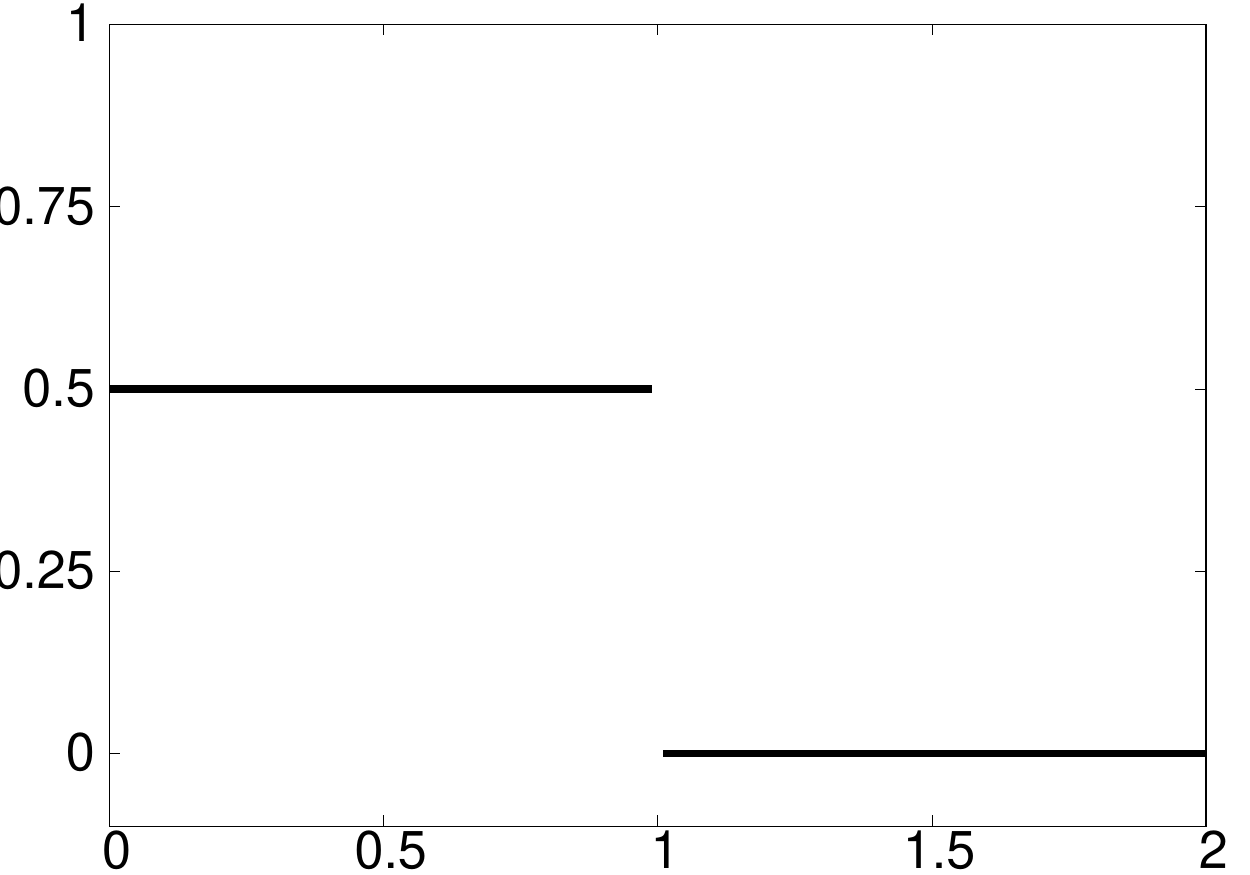}
&
\includegraphics[width=0.15\textwidth]{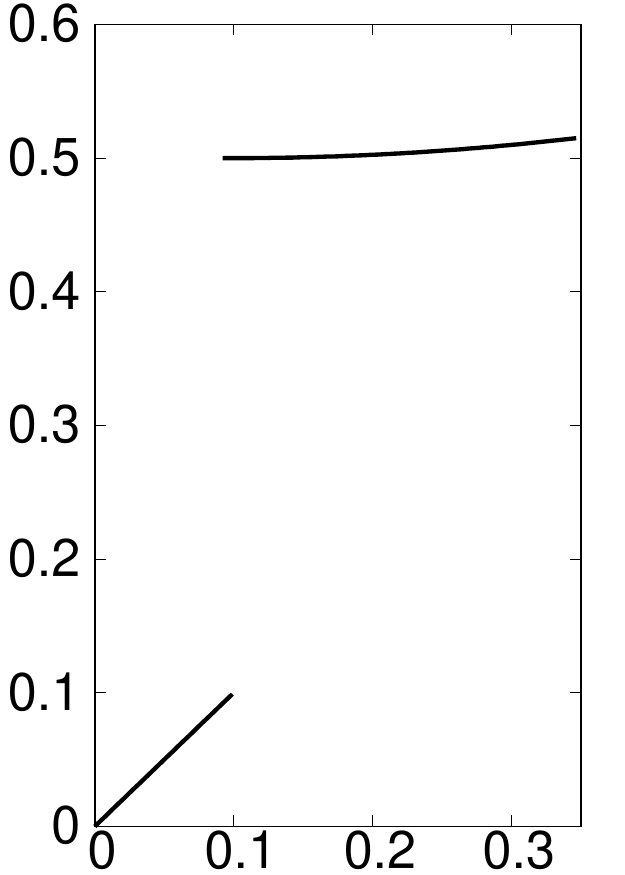}
\end{tabular}
\small
\begin{picture}(0,0)(400,0)
\put(50,5){$\tilde w$}
\put(110,5){$\hat w$}
\put(60,60){$K_{\tilde w}$}
\put(110,70){$K_{\hat w}$}
\put(110,100){$K$}
\put(250,5){$\tilde w$}
\put(305,5){$\hat w$}
\put(318,70){\tiny$R(K^2_\varepsilon)_X$}
\put(220,65){\tiny $R(K^1_\varepsilon)_X$}
\put(370,78){\tiny $K^2_\varepsilon:=K(\hat w,\varepsilon)$}
\put(370,88){\tiny $K^1_\varepsilon:=K(\tilde w,\varepsilon)$}
\put(98,-105){$w$}
\put(120,-20){$R(K_w)_X$}
\put(300,-105){$\varepsilon$}
\put(335,-20){$R(K_\varepsilon^2)_X$}
\end{picture}
\end{center}
\caption{Top-Left: the set $K=[0,1]^2 \cup \left( [1,2] \times \{\frac 12\}\right)$ and $K_w$ for two different measurements, $\tilde w\in (0,1)$ and $\hat w = 1.1 $; Bottom-Left: $R(K_w)_X$ as a function of $w$; Top-Right: the sets $K^1_\varepsilon$ and $K^2_\e$ and  their corresponding Chebyshev balls; Bottom-Right: Graph of $R(K^2_\varepsilon)_X$  as a function of $\varepsilon$.}
\label{fig1}
\end{figure}
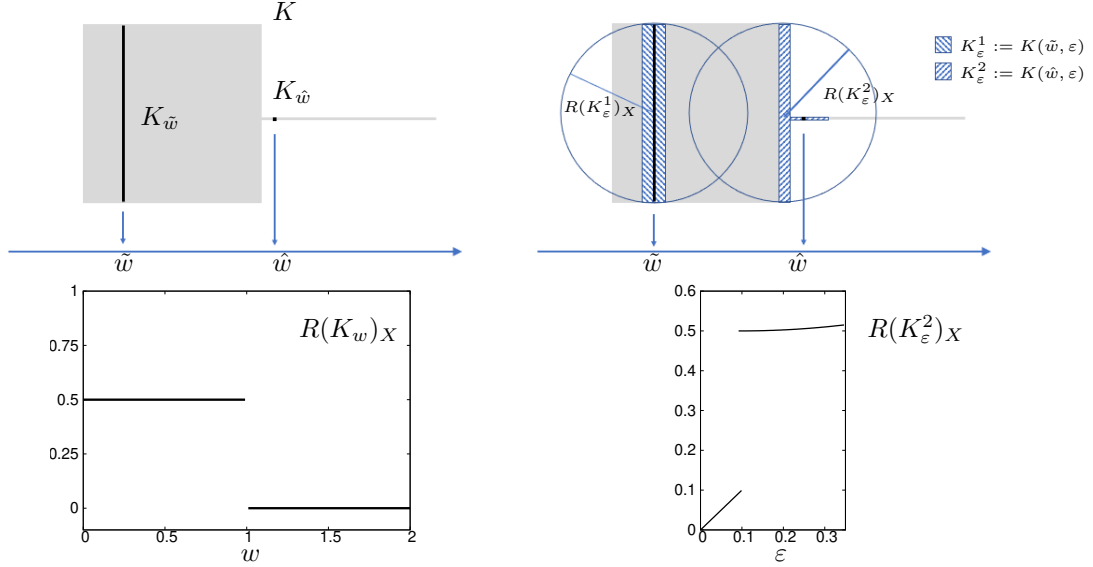
}

  For this example the function 
  $R(K_w)_X$ is a discontinuous function of $w$ (see the bottom-left graph
  in Figure \ref{fig1}). 
   Now we consider 
  the sets $K^1_\e:=K(\tilde w,\e)$ 
with $\tilde w\in (0,1)$ and $K^2_\e:=K(\hat w,\e)$ with $\hat w=1.1$, pictured 
  in Figure \ref{fig1} (top-right), 
  along with their corresponding Chebyshev balls.
  The graph of $R(K^2_\epsilon)_X$ as a function of $\e$ is presented at the bottom-right. This function is a discontinuous function of $\e$ with the discontinuity occuring at $\e=0.1$.
   If we move the point $\hat w$ to be closer to $1$, then the jump discontinuity in $R(K(\hat w,\e))_X$ as a function of $\e>0$ will  move
  closer to $0$. The main observation to make here is that the convergence of $R(K(\hat w,\e))_X$, $\e\to 0^+$, towards $R(K_{\hat w})_X=0$ is not uniform in $\hat w$ and depends on the distance of  $\hat w$ to $[0,1]$.

  This example shows that   obtaining quantitative bounds on the  performance  of numerical { procedures} via
  the construction of a $g_\e$  will be very much dependent on the set $K$ and will therefore need its own ad hoc analysis.

   \section{Near optimal recovery through discretization}
  \label{S:nosect} 
  
  The results of the previous section do not constitute a numerical learning { procedure}.  Rather, they only show that  optimal performance can be obtained if 
   an algorithm  provides a function $g_\e$ satisfying condition \eref{morereasonable}.  
    
  In this section, we begin our discussion of  { learning
   procedures} by formulating { finite-dimensional } optimization problems whose successful numerical implementation would yield  a near optimal numerical recovery 
  algorithm.  Thus, the problem of { designing  near optimal  learning
  procedures}  is reduced to questions centering around the convergence of optimization algorithms for the derived { finite-dimensional} optimization problem.

Any numerical { procedure}  for learning is based on some { method of } approximation.  The most common  { tools  used are}  polynomials, splines, wavelets, or neural networks.  Let $\Sigma_n$,
   $n=1,2,\dots,$ be the sets used for the approximation, where $n$ denotes the complexity of $\Sigma_n$.  The two main examples we have in mind
   are  the cases where $\Sigma_n$ is a linear subspace of $X$ of dimension $n$  and  the case where $\Sigma_n$ is  a parametric nonlinear manifold of functions from $X$ depending on
   $Cn$ parameters.  The most common example in the latter case is the nonlinear  manifold consisting of the outputs of a neural network (NN) with $n$ hidden neurons and  some
   specified architecture and activation function (see \cite{DHP} for an overview).  
   
   The first question  we address is how  we  should  use $\Sigma_n$ to build a numerical {procedure}.  The answer depends heavily on the structure of $K$ and is discussed in the sub-sections that follow.

   \subsection{Convex model classes}
   \label{SS:convex}
    We begin   with the most favorable case where
   $K$ is a compact convex centrally symmetric (about the origin) subset of $X$.    Any such set can be written as the unit ball $K=U(Y)$ of a normed linear
   subspace $Y$  of $X$, where the norm on $Y$ is induced by $K$  (see e.g. \cite{yosida}).    Since $K$ is compact, we have the embedding inequality
   \be
   \label{embed}
   \|g\|_X\le C_0\|g\|_Y,\quad g\in Y,
   \ee
   where $C_0$ is the embedding constant which depends only on $K$.
   {In order to simplify the notations, we use the convention $\|g\|_Y:=\infty$ for $g\notin Y$.}
 In this setting, we introduce
   for any $\mu>0$ the loss function 
   \be
   \label{loss1}
   \cL_\mu (g):= \|\lambda(g)-w\|+ \mu \|g\|_Y.
   \ee
   This loss function is  defined for all $g\in X$  { but infinite when $g\not \in Y$}.
   
   {We choose here not to raise the norms involved in the loss function to any powers but consider such variants in \S\ref{S:variants}.
     Formulation \eref{loss1} is also known as the  ``square-root'' formulation of the loss function since the data fidelity term is not squared. It is analogous to the so-called ``square-root LASSO'' decoder in statistics, compressed sensing and machine learning. 
   One of the motivations for introducing the square-root LASSO is that it is agnostic to the noise level, that is, the tuning parameter $\mu$ can be chosen 
independently of the norm of the noise, whereas in the standard LASSO the optimal tuning parameter should depend on the norm of the noise. These  decoders have recently become quite popular in learning problems.  While they were first popularized  in 
high-dimensional statistics (see \cite{Geer}), recently they were studied in the context of  (weighted) sparse recovery in compressed sensing (see \cite{Adcock1,Adcock3, Foucart, PJ}) and deep learning for high-dimensional function approximation (see  \cite{Adcock2}). }

   The following theorem describes a { finite-dimensional }  optimization problem whose solution is a near optimal recovery.
   \begin{theorem} 
   \label{T1:discrete}
    Let $K=U(Y)$ with $Y$ a normed linear subspace of $X$ and let the set 
    $\Sigma\subset X$ satisfy the condition
   \be
   \label{T11} \dist(K,\Sigma\cap K)_X
    < \delta.
   \ee
   Then,  if $f\in K_w$,  the function 
   \be
   \label{min0}
 \hat f:=  \hat f_{\Sigma,\mu} \in \argmin_{g\in \Sigma} \cL_\mu (g)
   \ee
 is a near optimal recovery of $f$, that is, 
   \be
   \label{nearopt}
   \|f-\hat f\|_X\le CR(K_w)_X, 
   \ee
   for any $C>2$,
     provided that
     $R(K_w)_X\neq 0$,  $\delta\le \mu^2$ and $\mu$ is sufficiently small.  More precisely, it is sufficient that 
     \be 
     \label{suffmu}
  \e:=  \mu\max (C_0,1+\mu)
     \ee
     satisfies the inequality
     $$
     \e+2 R(K(w,2\e))_X \le C R(K_w)_X,
     $$ 
which is possible for $\mu>0$
 small enough because of Lemma \ref{L:Rlimit}.
   \end{theorem}
   \begin{proof} 
   Let $f$ be any function in $K=U(Y)$ which satisfies the data, i.e., $f$ is in $K_w$, and let { $f_\Sigma \in \Sigma \cap K$ satisfy $\|f-f_\Sigma\|_X\le \delta$}. 
   From the definition of $\hat f$, we know that
  \be
  \label{minprop1} 
  \|w-\lambda( \hat f)\|+\mu \|\hat f\|_Y \le  \|w-\lambda({ f_\Sigma})\|+\mu \|{ f_\Sigma}\|_Y\le \delta+\mu,
  \ee
  where the first term was estimated by 
  $$\|w-\lambda({ f_\Sigma})\|= \|\lambda(f-{ f_\Sigma})\|\le \|f-{ f_\Sigma}\|_X\le\delta,
  $$
  and the second term uses that ${ f_\Sigma}\in K$ so that $\|{ f_\Sigma}\|_Y\le 1$.
  
  We now assume that  $\delta\le  \mu^2$ and $\mu$ is small. We see from \eref{minprop1} that $\hat f$ almost satisfies the data since
   $$\|w-\lambda(\hat f)\|\le \delta+\mu\le \mu(1+\mu).$$
    Also,  $\hat f$ is close to $K$ since \eref{minprop1} shows that  $\|\hat f\|_Y\le 1+\mu$. Therefore  
     $(1+\mu)^{-1} \hat f\in K$, and 
    from \eref{embed} we have
   $$ 
    \dist(\hat f,K)_X\leq
   \|\hat f - (1+\mu)^{-1} \hat f\|_X = \frac{\mu}{1+\mu}\|\hat f\|_X \le \frac{\mu}{1+\mu} C_0\|\hat f\|_Y\le \mu C_0.
   $$
      In other words, 
$g_\e:=\hat f$ satisfies \eref{morereasonable} for $\e:= \mu\max (C_0,1+\mu)$.   
  Theorem \ref{P:approxcenter} shows that   for any $C>2$, the function $\hat f$ is a near optimal recovery with constant $C$, provided
  $\mu$ (and hence $\delta$) is sufficiently small.  The last statement of the theorem follows from Theorem \ref{P:approxcenter}.
  \end{proof}
  
  \begin{remark}
  \label{R:numerror}
  In practice, numerical optimizers may not find a global minimizer $\hat f$ in \eqref{min0}. Rather, they are more likely to produce $\tilde f \in \Sigma$ such that
  $$
  \mathcal L_\mu (\tilde f) \leq \mathcal L_\mu(\hat f) + \tilde \e
  $$
  for some $\tilde \e>0$.
  In this case, the conclusion of Theorem~\ref{T1:discrete} remains valid provided $\tilde \e  $ is sufficiently small; namely $\tilde \e\le \delta\le \frac12\mu^2$.
 Indeed,  the estimate   \eqref{minprop1} gives
 $$
 \|w-\lambda( \tilde f)\|+ \mu  \|\tilde f\|_Y \le
 \delta+\mu +\tilde \e\le 2\delta+\mu
  \leq \mu(1+\mu),
 $$
  and the proof is completed as in the theorem.
   \end{remark}
   
    \begin{remark}
  Notice that if $\Sigma$ is convex (e.g. if $\Sigma$  is a linear space) and if $\|\cdot\|_Y$ is strictly convex, then the minimizer of $\mathcal L_\mu$ over $\Sigma$ is unique because
  $\mathcal L_\mu$ is always strictly convex on $\Sigma$.
  
    \end{remark}

  \subsection{General model classes}
  \label{SS:generalmodel}
  We next want to give a discretization problem whose solution is near optimal for any model class $K$, i.e., for  any compact set $K\subset X$.
  For this, we introduce the loss function
   \be
   \label{loss2}
   \cL_K (g):=  \|\lambda(g)-w\|+ \dist(g,K)_X, \quad  g\in X.
   \ee
 The following theorem holds.
  
  \begin{theorem} 
   \label{T2:discrete}
    Let $K$ be any compact subset of $X$ and let the set  $\Sigma\subset X$ satisfy the condition
   \be
   \label{T21} \dist(K,\Sigma)_X < \delta.
   \ee
   Then, if $f\in K_w$,   the function 
   \be
   \label{min1}
  \hat f \in \argmin_{g\in \Sigma}\cL_K(g)
   \ee
    is a near optimal recovery of 
    $f$, that is, we have for any $C>2$
   \be
   \label{nearoptimal1}
   \|f-\hat f\|_X\le CR(K_w)_X,
   \ee
    provided that $R(K_w)_X\neq 0$ and $\delta$ is sufficiently small.  More precisely, it is sufficient that 
   $$
   2\delta +2R(K(w,4\delta))_X \le CR(K_w)_X,
   $$
   which is possible because of Lemma \ref{L:Rlimit}.
   \end{theorem}
   \begin{proof} 
   Let $f$ be any function in $K$ which satisfies the data, i.e., $f$ is in $K_w$, and let { $f_\Sigma \in \Sigma$ satisfy $\|f-f_\Sigma\|_X\le \delta$}. 
   From the definition of $\hat f$, we know that
  \be
  \label{minprop} 
  \|w-\lambda( \hat f)\|+ \dist(\hat f,K)_X\le  \|w-\lambda({ f_\Sigma})\|+ \dist({ f_\Sigma},K)_X\le \delta+ \delta=2\delta,
  \ee
  where the first term was estimated by 
  $$\|w-\lambda({ f_\Sigma})\|= \|\lambda(f-{ f_\Sigma})\|\le \|f-{ f_\Sigma}\|_X\le\delta.
  $$
   It follows that the function $g_\e:=\hat f$ satisfies \eref{morereasonable} with $\e=2\delta$.  From Theorem \ref{P:approxcenter},
  we find that
   \be
   \label{T22}
   \|f-\hat f\|_X\le 2\delta +2R(K(w,4\delta))_X,\quad f\in K_w.
   \ee
   If  $C>2$ and   $\delta$ is sufficiently small,   Lemma \ref{L:Rlimit} gives that the right side of \eref{T22} does not exceed $CR(K_w)_X$.
   \end{proof}
   
   The difference between the case of a general model class $K$ and the special case where $K$ is convex and centrally symmetric is
   in the form of the penalty term in the loss function.  Ostensibly, the penalty in the general case would be more difficult to numerically implement.  Also note that the penalty term does not require a parameter to balance it with the data fitting term.  This is because we simply want both terms to be small simultaneously.
     
     \begin{remark}
     \label{R1}
     Let us emphasize that the results in this section do not yet give a numerical { procedure}  for near optimal recovery since
     we have not given a numerical recipe for solving the corresponding { finite-dimensional } problem. This is discussed in more details in \S\ref{S:examples}.  
     \end{remark}

  \begin{remark}
  \label{R:numerror1}
   Similarly to Remark \ref{R:numerror}, if we only approximately solve the minimization problem we still obtain a near optimal recovery provided
   the numerical error is small enough.
  \end{remark}

  \section{The special case of point values and recovery in $L_p$}
  \label{S:ptvalues}  
   We turn next to what is the most common setting in machine learning where the data comes from point evaluations.  We assume that $f$ is a function defined on  $\Omega\subset \R^d$, $d\ge 1$, where $\Omega$ is the closure of a bounded domain in $\R^d$.   For the moment, we assume  that we have noiseless observations
   \be
   \label{data1}
   w_i=f(x_i), \quad { x_i\in \Omega\subset \R^d},\quad i=1,\dots,m,
   \ee
   where the {\it data sites} $x_i$ come from $\Omega$.  The most common choice of the metric in which to measure recovery error is 
   an $X=L_p(\Omega)$ norm on $\Omega$ and $1\le p\le\infty$.    Since point evaluation is not a linear functional on $X$, we cannot apply the results of
   the previous section.  Note however, that to define  $\cL_\mu$ or $\cL_K$, it is enough to have point evaluation well defined for
   functions in the model class $K$ and functions in $\Sigma$. 
   
   In order to guarantee that point evaluation is well defined for functions in $K$, we  make the following assumption on the model class $K$.

   \vskip .1in

  \noindent
  {\bf Main Assumption:} {\it We assume the model class $K$ is a compact subset of $C(\Omega)$.  So, in particular,  all functions in $K$  have well defined point values.}
   
  \vskip .1in
  \noindent 
 Even though we impose this assumption on $K$, we continue to measure the performance of the learning {\nnew procedure}  in a Banach space  $X$ for which  we have an embedding
 \be 
 \label{Cembed}
 \|f\|_X\le C_X\|f\|_{C(\Omega)}.
 \ee 
   
   Let $K$ be a model class satisfying our {\bf Main Assumption}.  As in the previous section, we shall consider two settings depending on whether $K$ is convex and centrally symmetric, or $K$ is a general compact set.  The results we give in this section are similar to those of the preceding section
   with the modifications necessary to handle the new setting of point evaluation.  
   
   As before, let $K_w$ be the set of all $f\in K$ which satisfy the given measurements, i.e. \eref{data1}.   As in the previous section,  the optimal recovery for a model class $K$ with the data \eref{data1} is given by the Chebyshev center of $K_w$  and the optimal error is the  Chebyshev radius $R(K_w)_X$.  This radius will depend on $X$.
         
      Now, let us see what the previous section says since our setting is slightly different.     As before, we let
   \be
  \label{Rwe1}
  K(w,\e):= \bigcup_{\|w'-w\|\le \e} K_{w'},\quad 
  \ee
  and let $ R(K(w,\e))_X$ be the  Chebyshev radius of this inflated set in $X$.
  
   The following are the analogues of  Theorem \ref{P:approxcenter} and  Lemma \ref{L:Rlimit}. 
    We use the notation
   \be 
   \label{defpt}
   \lambda_{\bf x}(g):=(g(x_1), \ldots,g(x_m))\in \R^m, \quad {\bf x}:=(x_1,\ldots, x_m),\quad {\nnew
   x_i\in  \Omega\subset \R^d}, \quad g\in C(\Omega),
   \ee 
   when discussing point evaluation data.
   Note that when $g,h\in C(\Omega)$, it follows from definition \eref{l2} that
   $$
   \|\lambda_{\bf x}(g)-\lambda_{\bf x}(h)\|\leq \|h-g\|_{C(\Omega)}.
   $$

    \begin{prop}
  \label{P:approxcenter2}
  Let $X$ satisfy the embedding \eref{Cembed} and let $K$ satisfy our {\bf Main Assumption}.  If 
  $f\in K_w$, where  $w=\lambda_{\bf x}(f)$, and $g_\e$ is any function in $C(\Omega)$ with $\|w- \lambda_{\bf x}(g_\e)\| \le \e$ and $\dist(g_\e,K)_{C(\Omega)}\le \e$, then
  \be
  \label{P:a2}
 \|f-g_\e\|_X\le 
  C_X\e+ 2R(K(w,2\e))_X.
  \ee
  \end{prop}
  \begin{proof}  
  The proof is the same as that of Theorem \ref{P:approxcenter}  after choosing $h$ in $K$ to satisfy the inequality $\|g_\e-h\|_{C(\Omega)}\le\e$.
  \end{proof}
  
  \begin{remark}
  \label{R:counter}
  In the last proposition, it is natural to ask why we need $g_\e$ close to
  $K$ in the norm of $C(\Omega)$ and not just in the norm of $X$.
  The following example clarifies that issue.  Let $\Omega=[0,1]$ and $X=L_2[0,1]$. Consider a set $\bx$ of data sites. Let $K=\{f_1,f_2\}$ where $f_1 \equiv 1$ and
  $f_2 \equiv 0$.  Let $w=(1,1,\dots,1)$ which is data satisfied only by $f=f_1$.
  If $g_\e$ is one at the data and $\|g_\e\|_{L_2[0,1]}<\e$, then 
  $g_\e$ satisfies the data and is close to $K$ in the $X$ norm.  However,
  the left side of \eref{P:a2} is close to $1$ and the right side
  is close to $0$, so the Proposition using distance in $X$ is not valid. 
  \end{remark}

     \begin{lemma}
  \label{L:Rlimit2}  If $X $ satisfies \eref{Cembed}  and   $K$ satisfies our {\bf Main Assumption}, then,  we have
  \be
  \label{LR2}
  \lim_{\e\to 0^+} R(K(w,\e))_X=R(K_w)_X.
  \ee
  \end{lemma}
  \begin{proof} 
  The proof is similar to that of Lemma \ref{L:Rlimit} and so we only indicate the small differences. The  limit $R_0$ on the left in \eref{LR2} exists from monotonicity.  If $R_0>R(K_w)_X$, then  we fix $\xi>0$ such that 
$R_0>R_0-\xi>R(K_w)_X$ and because of the compactness of $K$ in $C(\Omega)$ there is a sequence $\e_n\to 0^+$ and a sequence $f_n\in K(w,\e_n)_X$  that converges to 
  a limit $f^*$ in $K$ in   the $C(\Omega)$ norm,  and because of \eref{Cembed}, in the $X$ norm.  Then, $\|f^*-z_w\|_{X}\ge  R_0-\xi>R(K_w)$.   From the convergence in $C(\Omega)$, it follows that $f^*$ is in $K_w$ and so we have  $\|f^*-z_w\|_{X}\le R(K_w)_X$, which contradicts that
  $\|f^*-z_w\|_{X}>R(K_w)$.  
  \end{proof}
  
  \bigskip
   Let us now  assume that $K=U(Y)$, where $Y$ is a subspace of $C(\Omega)$ equipped  with a norm $\|\cdot\|_Y$.  Typical examples for $Y$ are smoothness spaces: Lipschitz, Sobolev, Besov spaces.  The {\bf Main Assumption} is simply requiring that $Y$ compactly embeds into $C(\Omega)$. Therefore, we know that   
   $$
   \|f\|_{C(\Omega)}\leq C_Y, 
  \quad f\in K.
   $$
   We shall use this inequality as we proceed without 
   mentioning it.  Such embeddings typically follow from Sobolev embedding theorems.

   The following theorem describes a {\nnew finite-dimensional }  optimization problem whose solution is a near optimal recovery.  In the statement of this theorem we use the loss function $\cL_\mu$ as given in \eref{loss1} using point evaluation 
   functionals  (see \eref{defpt}).
   
   \begin{theorem} 
   \label{T3:discrete}
    Let $K=U(Y)$ with $Y$ a normed linear subspace of $C(\Omega)$  satisfy the {\bf Main Assumption}, let $X $ satisfy the embedding \eref{Cembed}, and let the set $\Sigma$ satisfy the condition
   \be
   \label{T31} \dist(K,\Sigma\cap K)_{C(\Omega)} < \delta.
   \ee
   If $f\in K_w$, where $w=\lambda_{\bf x}(f)$, 
  then the function 
   \be
   \label{min3}
  \hat f:=\hat f_{\Sigma,\mu} \in \argmin_{g\in \Sigma} \cL_\mu (g), \quad 
    \hbox{where}\quad \cL_\mu(g):=\|\lambda_{\bf x}(g)-w\|+\mu\|g\|_Y,
   \ee
    is a near optimal recovery 
    of $f$, that is,
   \be
   \label{optimal0}
   \|f-\hat f\|_X \le  CR(K_w)_X,
   \ee
  for any $C>2$,  
      provided that
       $R(K_w)_X\neq 0$, and $\delta$ and $\mu$ are sufficiently small.
      More precisely, it is enough to choose $\delta\le  \mu^2$
      and $\mu$ small enough so that   $C_X\e+2R(K(w,2\e))_X\le CR(K_w)_X$ with $\e:=\mu\max(\mu+1,C_Y)$.
      \end{theorem}
       
   \begin{proof} 
   The proof is the same as that of Theorem \ref{T1:discrete} except that now we choose ${\nnew f_\Sigma} \in \Sigma \cap K$ to satisfy $\|f-{\nnew f_\Sigma}\|_{C(\Omega)}\le \delta$ and use Proposition \ref{P:approxcenter2} and Lemma \ref{L:Rlimit2}.
 \end{proof}

  We also have the analogue to Theorem \ref{T2:discrete}.  
  
  \begin{theorem} 
   \label{T4:discrete}
    Let $K $  satisfy the {\bf Main Assumption}, let $X$ satisfy the embedding \eref{Cembed}, and let the set $\Sigma$ satisfy the condition
   \be
   \label{T111} \dist(K,\Sigma)_{C(\Omega)}
     < \delta.
   \ee
     If $f\in K_w$, where $w=\lambda_{\bf x}(f)$, 
  then the function 
   \be
   \label{min11}
  \hat f:=\hat f_{\Sigma} \in \argmin_{g\in \Sigma} \cL'_K (g), \quad 
    \hbox{where}\quad \cL'_K(g):=\|\lambda_{\bf x}(g)-w\|+\dist(g,K)_{C(\Omega)},
   \ee
    is a near optimal recovery 
    of $f$, that is
   \be
   \label{optimal1}
   \|f-\hat f\|_X\le CR(K_w)_X,
   \ee
 for any $C>2$, provided $R(K_w)_X\neq 0$ and $\delta$ is  sufficiently small. More precisely, it is sufficient that
   $$
   C_X\e+2R(K(w,2\e))_X\leq CR(K_w)_X, \quad \hbox{where}\quad 
   \e:=2\delta.
   $$
   \end{theorem}
   
 \begin{proof} 
 The proof is the same as that of Theorem \ref{T2:discrete} except that now we choose 
    ${\nnew f_\Sigma}\in \Sigma $ to satisfy $\|f-\AB{f_\Sigma}\|_{C(\Omega)}\le \delta$ and use Proposition \ref{P:approxcenter2} and Lemma \ref{L:Rlimit2}. 
 \end{proof}

  \section{Noisy measurements}
  \label{S:noisy}
  In this section, we consider the case when the measurements are corrupted by an additive deterministic noise.  
  Namely, we assume that our measurements are now given by
  \be
  \label{nm}
  \tilde w_j=f(x_j)+\eta_j,\quad j=1,\dots,m,
  \ee
 where the real numbers $\eta_j$, $j=1,\dots,m$, are unknown to us.  However,
 to derive quantitative results on performance, we will have to make some
 assumptions on the unknown noise vector ${\boldsymbol \eta}:=(\eta_1,\dots,\eta_m)$.  We assume that  all we know about ${\boldsymbol \eta}$ is its size.
 
We continue to let $w_j=f(x_j)$, $j=1,\dots, m$,  and $w:=(w_1,\dots,w_m)$.   We put ourselves in the
same setting as in the previous section where $f\in K$ and $K$ satisfies our {\bf Main Assumption}.  We let $K_w$ again be the set of 
$f\in K$ such that $f(x_j)=w_j$, $j=1,\dots,m$, and continue to use the inflated sets $K(w,\e)$, $\e>0$.

We assume that we have a bound on the noise vectors of the form
\be
\label{bnoise}
\|\boldsymbol \eta\|\le \gamma <\infty.
\ee
     Then, the  totality of information we have  
about $f$ is that $f\in K$ and $f$ satisfies the data $\tilde w-{\boldsymbol \eta}$ where $\tilde w$ is our observation vector and
 $\|\boldsymbol\eta\|\le \gamma$.   It follows that the totality of information we have about $f$ is that it is in the
set $K(\tilde w,\gamma)$.   This means that the error of optimal recovery of $f$ from such noisy observations is given by
\be
\label{bestnoisy}
{\bf best \ noisy \  rate} =R(K(\tilde w,\gamma))_X.
\ee

     We formulate recovery results for any set $K$ that satisfies our {\bf Main Assumption}. 
For any function $g\in C(\Omega)$, we continue to use the notation $\lambda_{\bf x}(g):=(g(x_1),\dots,g(x_m))$, where ${\bf x}=(x_1,\ldots,x_m)\in \Omega$.  
To recover $f$ from the noisy observations $\tilde w$, we use the loss function 
\be
\label{noisylossP}
 \cL_{K,\tau}(g):=\tau \|\tilde w- \lambda_{\bf x}(g)\| +\dist(g,K)_{C(\Omega)},
\ee
 where the parameter $\tau$ is a positive real number.

\begin{theorem} 
   \label{T5:discretenoisyP}
    Let $K$  satisfy the {\bf Main Assumption}, let $X$ satisfy the embedding \eqref{Cembed}, and let the set $\Sigma$ satisfy
    the condition 
   \be
   \label{T11P} \dist(K,\Sigma)_{C(\Omega)}< \delta.
   \ee
     Consider the function 
   \be
   \label{min11P}
 \hat f:= \hat f_{\Sigma,\tau} \in \argmin_{g\in \Sigma}\   \cL_{K,\tau}  (g),
 \quad  \hbox{where}\quad 
 \cL_{K,\tau}(g):=\tau \|\tilde w- \lambda_{\bf x}(g)\| +\dist(g,K)_{C(\Omega)},
   \ee
where $\tilde w$ are the noisy data observations of a function $f\in K$ with an unknown noise vector
   $\boldsymbol \eta$ which satisfies $\|\boldsymbol \eta\|\le \gamma$ for some finite number $\gamma\le 1$ considered unknown.
   Then,  $\hat f$ is a near optimal recovery of $f$, that is,
   \be
   \label{noisyerrorP}
   \|f-\hat f\|_X\le C R(K(\tilde w, \gamma))_X,
   \ee
    for any $C>2$, provided   
    $R(K(\tilde w, \gamma))_X\neq 0$, 
    $\delta\le \tau^2$ 
    and 
  $\tau$ is  sufficiently small.
   \end{theorem}
    \begin{proof}  
    We   assume $\tau<1$ and let  $f\in K_w$ 
     and  $\tilde w$ be the noisy observations of $f$  with noise vector $\boldsymbol \eta$ satisfying  $\|\boldsymbol \eta\|\le  \gamma$.  Let ${\nnew f_\Sigma}\in\Sigma$
     satisfy $\|f-{\nnew f_\Sigma}\|_{C(\Omega)}\le \delta$.  Then, we know that 
     $\|\lambda_{\bf x}({\nnew f_\Sigma})-w\|\le\delta$ and so   $\|\lambda_{\bf x}({\nnew f_\Sigma})-\tilde w\|\le \delta+\gamma$.   Since ${\nnew f_\Sigma}\in \Sigma$,
     we have   
  $$
  \tau \|\tilde w-\lambda_{\bf x}( \hat f)\| + \dist(\hat f,K)_{C(\Omega)} \le  \tau \|\tilde w-\lambda_{\bf x}({\nnew f_\Sigma})\| +\dist({\nnew f_\Sigma},K)_{C(\Omega)} \le \tau(\delta+\gamma)+\delta < \tau\gamma+ 2\delta.
 $$
    It follows   that  for 
    $\delta\le \tau^2$
     we have
   \be
   \label{followsthatP}
   \|\tilde w-\lambda_{\bf x}(\hat f)\|<  \gamma + 2\frac{\delta}{\tau} \leq \gamma+2\tau \quad {\rm and} \quad \dist(\hat f,K)_{C(\Omega)}\le  \tau\gamma+2\delta.
   \ee
  Now let $h\in K$ satisfy 
  \be
  \label{satisfy1}
  \|\hat f- h\|_{C(\Omega)}\le \tau\gamma+2\delta
   \leq \tau\gamma+2\tau^2
   {\nnew =\tau(\gamma+2\tau)< 3\tau}.
  \ee
  Then, we have
\begin{eqnarray*}
  \|\tilde w-\lambda_{\bf x}(h)\|
  &\le &\|\tilde w-\lambda_{\bf x}(\hat f)\|+\|\lambda_{\bf x}(\hat f)- \lambda_{\bf x}(h)\|\le \gamma+2\tau+\|\hat f-h\|_{C(\Omega)}\\
  \nonumber
  &\le &  \gamma+2 \tau+ \tau\gamma+2\tau^2<  \gamma +\tau(4+\gamma)
 \leq \gamma+5\tau.
\end{eqnarray*}

So, $h\in K(\tilde w,  \gamma+5\tau)$
  and so is $f$.  We therefore obtain 
  $$
  \|f-h\|_X\le 2 R(K(\tilde w, \gamma+5\tau))_X.
  $$
  Finally, from  the embedding inequality \eref{Cembed} and \eref{satisfy1}, we have
  \begin{eqnarray}
 \nonumber
  \|f-\hat f\|_X&\le& {\nnew \|f-h\|_X+\|h-\hat f\|_X
 \leq
 \|f-h\|_X+C_X\|h-\hat f\|_{C(\Omega)}}
 \\
 &\leq &
 {\nnew 2 R(K(\tilde w,  \gamma+5\tau ))_X+ 3C_X\tau.}
  \label{finally}
  \end{eqnarray}
  If we  {\nnew use Remark \ref{R:genconvergence}, we obtain \eref{noisyerrorP} 
  for any $C>2$, provided we take $\tau$ suitably small, see \eref{limit1}}.
  \end{proof}

      \begin{remark}
\label{R:whytau}
 The appearance of the parameter $\tau$ in the case of noisy observations is
 quite natural since the confidence in the measurements decreases as the noise level increases.  When there is no noise, we have complete confidence in the measurements and so we can take $\tau=1$ as was done in the previous section. \end{remark}

      \begin{remark}
\label{R:choosetau}
 Note that in the above, it is not necessary to know either $\gamma$ or $\tau$ in order to arrive at the inequality \eref{finally}.  However, to guarantee that the recovery is near optimal, one needs to  choose $\tau$ (and hence $\delta$)  sufficiently small depending on the nature of $K$.  
 \end{remark}

\begin{remark}
\label{R:stochastic}
 The most common setting for noise in statistics is to assume
 that the noise vector is composed of independent random draws
 with respect to an underlying probability distribution. Optimal performance in such a setting is referred to as minimax rates.  We do not treat this case in this paper since it requires some substantially new ideas. 
 \end{remark}

 \section{Variants} 
 \label{S:variants}
 
 This section considers variants of the minimization problems already discussed and emphasises certain aspects of these problems that are useful in numerical implementation. 
Since the treatment of the other cases is similar, we concentrate on the model assumption $f\in K=U(Y)$ and the loss function $\cL_\mu$ given in \eref{loss1}. An alternative to $\cL_\mu$ is the loss function
 \be 
 \label{lossmodified}
 \cL'_\mu(g):= \|w-\lambda(g)\|^\alpha+\mu \|g\|_Y^\beta,
 \ee 
where $\alpha>0$ and $\beta>0$ are fixed. 
{\nnew
The modified loss function \eqref{lossmodified} generalizes the original loss function \eqref{loss1} by raising the data fidelity and regularization term to possibly different powers. The case $\alpha=\beta=2$ is particularly appealing when $Y$ is a Hilbert space as it leads to a simple expression for the derivative of $\cL'_\mu$. Also, the modified loss function can be viewed as a generalization of the classical LASSO procedure, which corresponds to $\alpha=2$ and $\beta=1$. In some 
special settings, such loss functions have been considered
before in the context of recovery of sparse signals, see \cite{Foucart, PJ} and the references therein.
}
\begin{remark}
 \label{R:variant}
 If $\Sigma$ is convex (e.g. if it is  a finite dimensional linear space),  if $\| . \|_Y$ is strictly convex, and if $\alpha, \beta \geq 1$, then the minimizer of $\cL'_\mu(g)$ over $g\in \Sigma$ is unique.   In this case, if 
 $\Sigma$  is a finite dimensional linear space, the solution $\hat f$ of the resulting optimization problem \eref{var1} 
  can be computed by available optimization algorithms. 
 Non-uniqueness can occur for other settings, for example when the second term is a quasi-norm or some other nonconvex regularizer. The latter are sometimes preferred due to their better performance in special cases. 
 
 \end{remark}

Following the ideas from Subsection \ref{SS:convex}, we  establish similar near optimality  results for this loss function
{\nnew in the case where the measurements are linear functionals on $X$.}

   \begin{theorem} 
   \label{T1:discreteAlt}
    Let $K=U(Y)$ with $Y$ a normed linear subspace of $X$ and let the set $\Sigma$ satisfy \eref{T11}.
   Then, for any $C>2$, the function  
   \be 
   \label{var1}
   \ds\hat f:=  \hat f_{\Sigma,\mu} \in \argmin_{g\in \Sigma} \cL'_\mu (g)
   \ee 
   is a near optimal recovery, i.e., 
   \be
   \label{nearoptAlt}
   \|f-\hat f\|_X\le CR(K_w)_X, \quad f\in K_w,
   \ee
provided 
  $R(K_w)_X\neq 0$,       $\delta^\alpha\le \mu^2$ and $\mu$ is sufficiently small.
     Moreover, in the case of numerical optimization producing $\tilde f$ such that 
     $\mathcal L'_\mu (\tilde f) \leq \mathcal L'_\mu(\hat f) + \epsilon$
  for some $\epsilon>0$, 
 the estimate \eref{nearoptAlt} remains valid with $\hat f$ replaced by $\tilde f$ and provided that $\epsilon\le \delta^\alpha\le \frac12\mu^2$ and $\mu$ is sufficiently small.
   \end{theorem}
   \begin{proof} 
   Let $f$ be any function in $K=U(Y)$ which satisfies the data, i.e., $f$ is in $K_w$, and let ${\nnew f_\Sigma}\in \Sigma \cap K$ satisfy $\|f-{\nnew f_\Sigma}\|_X\le \delta$. 
   From the definition of $\hat f$, we know that
  \be
  \label{minprop1Alt} 
  \|w-\lambda( \hat f)\|^\alpha+\mu \|\hat f\|_Y^\beta \le  \|w-\lambda({\nnew f_\Sigma})\|^\alpha+\mu \|{\nnew f_\Sigma}\|_Y^\beta\le   \delta^\alpha+\mu,
  \ee
  where the first term was estimated by 
  $$\|w-\lambda({\nnew f_\Sigma})\|= \|\lambda(f-{\nnew f_\Sigma})\|\le \|f-{\nnew f_\Sigma}\|_X\le\delta,
  $$
  and the second term uses that ${\nnew f_\Sigma}\in K$ so that $\|{\nnew f_\Sigma}\|_Y\le 1$.
  
  We now assume that  $\delta^\alpha\le  \mu^2$ and $\mu$ small. We see from \eref{minprop1Alt} that $\hat f$ almost satisfies the data since
   $$\|w-\lambda(\hat f)\|^\alpha\le \delta^\alpha+\mu\le \mu(1+\mu).$$
    Also,  $\hat f$ is close to $K$ since \eref{minprop1Alt} shows that  $\|\hat f\|_Y^\beta\le 1+\mu$ and so $ (1+\mu)^{-\frac1\beta} \hat f\in K$ and from \eref{embed} we have
   $$ \|\hat f - (1+\mu)^{-\frac1\beta} \hat f\|_X = (1-(1+\mu)^{-\frac1\beta})\|\hat f\|_X \le (1-(1+\mu)^{-\frac1\beta}) C_0\|\hat f\|_Y\le  ((1+\mu)^{\frac1\beta}-1)C_0.
   $$
      This means that  $g_\e:=\hat f$ satisfies \eref{morereasonable} for $\e:= \max \Big(C_0((1+\mu)^{\frac1\beta}-1),(\mu^2+\mu)^{\frac1\alpha}\Big)$.   
  Theorem \ref{P:approxcenter} shows that   for any $C>2$, the function $\hat f$ is a near optimal recovery with constant $C$ provided
  $\mu$ (and hence $\delta$) is sufficiently small.
  
  In the case of a numerical approximation $\tilde f$ to  $\hat f$,  the estimate   \eqref{minprop1Alt} gives
 $$
 \|w-\lambda( \tilde f)\|^\alpha+ \mu  \|\tilde f\|^{\beta}_Y \le
 \delta^\alpha+\mu +\epsilon\le 2\delta^\alpha+\mu\le\mu^2+\mu,
 $$
  and the proof is completed as above.
  \end{proof}

  \section{Sampling rates}
  \label{S:sampling} 
  Although this is not the main topic of this paper, an important issue
  in learning is how many samples $m$ are needed to guarantee that an $f\in K$
  can be learned with a prescribed accuracy. In this section, we mention three concepts that give
  a benchmark for the accuracy issue.  We refer to these concepts in the next section where we discuss what our results say
  in two common settings for model classes in learning.
    
  So far, we have discussed learning primarily from the viewpoint that we were given data and wish to recover the function $f$ which gave rise to this data.
  In that setting, we had no role in the choice of the data sites.  A natural question is if we are given a budget $m$ of samples we can take of a function $f\in K$,    what would be the best choice of data sites.  Historically, there are three concepts that address this issue: Gelfand widths, sampling numbers, and
  averaged sampling numbers.  We briefly introduce these notions in this section.
  
  \subsection {Gelfand widths}
  \label{SS:Gelfand}
  Suppose that $K$ is a compact set in the Banach space $X$ and we are allowed to use our knowledge of $K$ to introduce $m$ sampling functionals $\lambda_1,\dots,\lambda_m$ to use in
  sampling the elements of $K$. Which functionals should we choose and what is the accuracy at which we could recover any $f\in K$ from the data $\lambda_1(f),\dots,\lambda_m(f)$?  The Gelfand width 
  \be 
  \label{Gelfandwidth}
  d^m(K)_X:= \inf_{\lambda_1,\dots,\lambda_m\in X^*} \sup_{f\in K} R(K_{\lambda(f)})_X, \quad \hbox{where}\quad \lambda(f):=(\lambda_1(f),\dots,\lambda_m(f)),
  \ee 
  is the optimal accuracy we can achieve in the worst case sense.  
  
  The Gelfand widths of model classes $K$ are a well studied concept in 
  Functional Analysis and Approximation Theory (see e.g. the book of Pinkus \cite{P}).  The Gelfand widths of classical model classes $K$ in classical Banach spaces $X$ are for the most part known and the Gelfand widths of novel model classes proposed in modern learning are currently being investigated
  (see e.g. \cite{SX,PV}).   Let us also note that Gelfand widths were the
  origins  of compressed sensing which studies the encoding and decoding
  of signals $f$ from a model class $K$ described by sparsity.  There it is shown that a random choice of $\lambda_1,\dots,\lambda_m\in X^*$ is with high probability near optimal (see \cite{Do,CDD} and the many books written on compressed sensing such as \cite{FR}). 
  
  For us, the Gelfand width $d^m(K)_X$ gives a lower bound for the 
  accuracy with which we can recover a general $f\in K$ from linear measurements of $f$.  A general criticism of  the concept of Gelfand widths is that
   in practical applications of sampling, one does not have access to arbitrary chosen general linear functionals of the target signal. Instead, the available functionals are more restricted.   For this reason, one typically imposes restrictions on the functionals $\lambda_j$, $j=1,\dots,m$.  If one requires that the sampling is done via point evaluation of $f$, then this leads to the concept of
  sampling numbers.
  
  \subsection{Sampling numbers}
  \label{SS:sampling}
   Let $K$ be  a subset of $C(\Omega)$ with $\Omega$ the closure of a bounded domain
   in $\R^d$.  If we restrict the linear functionals used as data observations to be point values of $f$, 
   then the optimal performance of $m$ such samples of $f$ is given by
   \be 
   \label{sn}
   s_m(K)_X:=\inf_{x_1,\dots,x_m\in\Omega} \sup_{f\in K} R(K_{f(\bx)})_X,\quad
   f({\bf x}):=(f(x_1), \ldots,f(x_m)),\quad m=1,2,\dots.
   \ee 
   The points $x_1,\dots,x_m$ that give the infimum in $s_m(K)_X$ are the optimal sampling sites.
   For spaces $X$ for which point evaluations are linear functionals, we obviously have $s_m(K)_X\ge d^m(K)_X$ and the difference in these two numbers is often substantial.  Sampling numbers are well studied for
   classical model classes $K$ in classical Banach spaces $X$, especially in the Information Based Complexity community, where it is referred to as standard information.  However, for novel model classes of functions of many variables
   that arise in modern learning there are many open questions on the asymptotic
   decay of the sampling numbers as $m\to\infty$.

   \subsection{Average sampling}
   \label{SS:avesampling}
   It is sometimes difficult to determine the sampling numbers of  a model class $K$ and even more so the position of the optimal data sites. In this case, one  studies the expected performance
   when the data sites $\bx$ are chosen randomly with respect to a probability measure 
   {\color{red} $\rho$} on $\Omega$. The relevant measure of performance is the averaged
   sampling numbers given by
   \be 
   \label{asn}
{ \bar s_m(K,{\nnew\rho})_X:=  
{\rm Exp}_{\bf x}\,\sup_{f\in K}R(K_{f(\bx)})_X.}
   \ee

  \section{Examples}
  \label{S:examples}
  
  The main objective of this paper is to describe the optimal
  performance that is possible for a learning {\nnew procedure}  and to show that this optimal performance can be achieved by solving
  an over-parameterized optimization problem.  In this sense, we
  provided a justification for the use of over-parameterized optimization which is now  a common staple in machine learning.
  Exactly how this plays out in practice depends very much on the model class $K$ which gives the properties of the  function $f$ to be learned. 
  
  Two natural questions arise in the numerical implementation  of this theory.
  The first is how fine must we take $\Sigma_n$ and how to chose the parameters in the loss function in order to guarantee
  near optimal learning.  The second question is to describe a numerical method with convergence guarantees for solving
  the resulting {\nnew finite-dimensional } optimization problem.

  We know from the exposition given above that when given a model class $K$ and linear data observations of an $f\in K$ that the optimal
 accuracy in recovering $f$ from these observations is $R(K_w)_X$
 and that a near optimal recovery is given by solving a {\nnew finite-dimensional }
 over-parameterized optimization problem.  The amount of over-parameterization necessary depends on $R(K(w,\e))_X
 $ and how fast it converges to $R(K_w)_X$ as $\e\to 0^+$.  This in turn depends very much on the particular $K$ and requires an ad-hoc analysis depending on $K$.  
 {\nnew  We describe a typical  way to proceed in the setting of 
 Theorem \ref{T1:discrete}.  The first step is to  construct a  sequence of spaces $\{ \Sigma_n \}_{n}$ for which $\dist(K,\Sigma\cap K)_X\lsim n^{-r}$ for some $r>0$.  The next step is to provide bounds for the Chebyshev radius
  and the inflated Chebyshev radii corresponding to the data $w=\lambda(f)\in \R^m$.   Typically we prove an estimate like
 $R(K(w,\e))_X\leq R(K_w)_X+\e$. In that case, according to the theorem, we   need 
 $n\gsim [R(K_w)_X]^{-2/r}$
  and a value $\mu\asymp n^{-r/2}$ in the loss function \eref{loss1} to guarantee that $\hat f$ is a near optimal recovery of $f$.
 }
 
 In order to illustrate what is involved in such an analysis, we discuss two examples in this section.  There are numerous other examples that could be considered and would be relevant to what is done in current practice of machine learning.

  \subsection{Point values of a smooth function}
  \label{SS:classical}

  A traditional setting in learning is to consider the data to be point evaluations of  a function $f$ defined on a domain $\Omega\subset \R^d$ and to  measure the error of  recovering $f$  in an $L_q(\Omega)$ norm, $1\le q\le \infty$. This is an extensively studied setting in IBC.  The texts \cite{TW,NW} are general references for this case. Our goal in this section is  to shine a light on
  what the results of the present paper have to say about optimal learning in this setting. For simplicity of discussion,   we assume $\Omega:=[0,1]^d$ and $q=2$; the extension to $q\neq 2$
  and more general domains can be found for example in
  \cite{KNS} and the references in that paper.
  
  For our model classes, we consider the unit ball
  $K:=U(W^s(L_p(\Omega)))$, $s>0$, $1< p\le  \infty$, of the Sobolev space $W^s(L_p(\Omega))$.
  In order to have $K$ a compact subset of $C(\Omega)$, we assume
  $s>d/p$.  The results mentioned in this section generalize to
  the case when $\Omega$ is a bounded Lipschitz domain and
  the Sobolev space is replaced by a more general Besov space
  as long as we continue to have a compact embedding into $C(\Omega)$.
  
  Let $x_j\in\Omega$, $j=1,\dots, m$, be $m$ data sites and 
 $w_j=f(x_j)$,
 $j=1,\dots,m$,
  be data observations of an $f\in K$.  We take these measurements to be exact;  noisy measurements can be treated 
  as discussed in \S \ref{S:noisy}.
     We use our notation $\bx=(x_1,\dots,x_m)$ for the data sites.  
  
  The optimal recovery error $R(K_w)_X$,  $X=L_2(\Omega)$, depends on the position of the data sites as is described for example in
  \cite{NWW,KS}.   It is known that  near optimal sampling sites $\bx$ are those that are uniformly spaced and the optimal recovery error $R(K_w)_{L_2(\Omega)}$ in the case of uniform spacing is $\approx m^{-s/d+ (1/p-1/2)_+}$.  For more general positioning of the point $\bx$, the optimal recovery rate is also known and depends on the maximal distance between the points of $\bx$ (see \cite{NWW}).
  Additionally, it is known  that
  $m$ random sample sites are near optimal save for a possible logarithm \cite{KNS}. {\nnew Procedures}  for near optimal
  recovery are known   using quasi-interpolants (see  \cite{KNS}).   
  
   Our results show that a near optimal recovery can be obtained
   by choosing a sufficiently fine linear or nonlinear space
   $\Sigma=\Sigma_n$, and solving the penalized least squares problem
   \be
   \label{loss11}
   \hat f'_\Sigma:=\argmin_{S\in\Sigma}
   \left[\Big[\frac{1}{m}\sum_{j=1}^m[w_j-S(x_j)]^2\Big]^{1/2}+\mu \|S\|_{W^s(L_p(\Omega))}\right ], 
   \ee
   with $\mu$ chosen sufficiently small.  According to \S \ref{S:variants},  we may also obtain near optimal performance by using the modified loss
   \be
   \label{loss113}
   \hat f_\Sigma:=\argmin_{g\in\Sigma}
\left [\frac{1}{m}\sum_{j=1}^m[w_j-g(x_j)]^2+\mu \|g\|^p_{W^s(L_p(\Omega))}\right ], 
   \ee
   with $\mu$ chosen sufficiently small.  This latter loss is convenient for numerical implementation as discussed below.
   There are several natural choices for $\Sigma$ such as a linear FEM space or a linear space spanned by B-splines or wavelets.

 \subsubsection{Analysis}  
   {\nnew 
   
   As a starting point for our analysis, let us recall the following known lemma.
   
   \begin{lemma}
       \label{L:pwl}  Let $1\le p\le\infty$, $n\ge 2$, and $0=\xi_1<\xi_2<\cdots<\xi_n=1$. Then given any $f$  with $f'\in L_p[0,1]$, the piecewise linear function $S$ which interpolates $f$ at
       the points $\xi_1,\dots,\xi_n$, and has breakpoints only at these points,  satisfies the inequalities:
\vskip .1in
       \noindent 
       {\rm (i)} \ $\|S'\|_{L_p[0,1]}\le \|f'\|_{L_p[0,1]}$,
       \vskip .1in
       \noindent 
       {\rm (ii)} $\|f-S\|_{L_2[0,1]}\le \|f'\|_{L_p[0,1]}h^s, \quad where\  h:=\max_{1\le j<n} |\xi_{j+1}-\xi_j|$\  and  \ 
   $s:=1-(1/p-1/2)_+$.
     \end{lemma}
   \begin{proof}  For notational convenience, we assume $p<\infty$.  The same proof holds when $p=\infty$. Let $I_j:=[\xi_j,\xi_{j+1}]$ and let $\mu_j$ be the slope of $S$ on $I_j$, Then, we have
   \be 
   \label{slopes}
   |\mu_j|=\frac{1}{|I_j|}\left|\int_{I_j}f'\right|\le |I_j|^{-1} \left[\int_{I_j}|f'|^p\right]^{1/p}|I_j|^{1-1/p} =|I_j|^{-1/p} \left[\int_{I_j}|f'|^p\right]^{1/p},\quad 1\le j<n,
   \ee 
   and thus
   $$
   |\mu_j|^p\leq |I_j|^{-1}\int_{I_j}|f'|^p.
   $$
    This means that $\int_{I_j}|S'|^p\le \int_{I_j}|f'|^p$, $1\le j<n$, and (i) follows. 
    
    To prove (ii), we note that $f-S$ vanishes
    at each $\xi_j$, $1\le j\le n$,  and therefore we have
    \be 
    \label{Lpnorms}
 \|f-S\|_{L_\infty(I_j)}\le \int_{I_j}|f'-S'|\le \int_{I_j}|f'|\le \[\int_{I_j}|f'|^p\]^{1/p} |I_j|^{1-1/p},\quad 1\le j< n.
 \ee 
 Here, we used the fact that on $I_j$, $S'=\mu_j=\frac{1}{|I_j|}\int_{I_j}f'$ is the best $L_1(I_j)$ approximation to $f'$ by constants.
Using \eref{Lpnorms}, we obtain
 \be 
    \label{Lpnorms1}
 \|f-S\|^2_{L_2[0,1]}=\sum_{j=1}^{n-1}\int_{I_j}|f-S|^2\le   \sum_{j=1}^{n-1}\left[\int_{I_j}|f'|^p\right]^{2/p} |I_j|^{2-2/p+1}\le h^{3-2/p} \sum_{j=1}^{n-1}\left[\int_{I_j}|f'|^p\right]^{2/p}.
 \ee 
If $p\le 2$, the last sum is bounded by $\|f'\|^2_{L_p[0,1]}$ because an $\ell_{2/p}$ norm is bounded by an $\ell_1$ norm,
and we arrive at
$$
\|f-S\|^2_{L_2[0,1]}\le h^{3-2/p}\|f'\|^2_{L_p[0,1]}.
$$
This gives (ii) when $p\le 2$.  The case $p\ge 2$ in (ii) follows from the case $p=2$.
\end{proof}

\begin{remark}
\label{R:Linf}
    Notice that from the bound \eref{Lpnorms} we get $\|f-S\|_{L_\infty[0,1]}\le h^{1-1/p}\|f'\|_{L_p[0,1]}$.
    We will use this inequality later in this paper.
\end{remark}
    
{\nnew
We  take $X=L_2[0,1]$ as the space in which we measure the error of
   performance. 
   To describe in a bit more detail one simple example, we consider the univariate
   case $d=1$ and 
   $$
   K=\{f\in W^1(L_p[0,1]) \ : \ \|f\|_{W^1(L_p[0,1])}\leq 1\}, \quad 1< p\le  \infty,
   $$
   where
   $$
   \|f\|_{W^1(L_p[0,1])}:=\max\{\|f\|_{L_p[0,1]},\|f'\|_{L_p[0,1]}\}.
   $$
   }
   Let us return to our problem of near optimal recovery of a function in $K$ from its point values $w_j=f(x_j)$, $j=1,\dots,m$. For convenience, we assume that the endpoints $0,1$ are always data sites and $w_1=0$. Note that in this case
   $$
 K_w=\{f\in W^1(L_p[0,1]):\,f(x_j)=w_j, j=1, \ldots,m, \,\|f'\|_{L_p[0,1]}\leq 1\}.
   $$ Let
   $\Xi_n$ be the union of the points sets $\{x_1,\dots,x_m\}$ and $\{0,1/n,\dots,1\}$ and let $\Sigma_n$ be the linear space of piecewise linear functions subordinate
   to $\Xi_n$.  Thus, we are in the setting of Theorem \ref{T3:discrete}.  The above remark tells us that if we choose $\Sigma=\Sigma_n$ 
    with $n$ large enough then
   $\Sigma$ satisfies
   \be 
   \label{errorrate}
   \dist(K,\Sigma_n\cap K)_{C[0,1]}\le n^{-1+1/p},\quad n\ge 1.
   \ee
   This means that the hypothesis of Theorem \ref{T3:discrete} are satisfied for $p>1$ and so solving the optimization problem \eref{min3} gives us a near optimal recovery provided $n$ is sufficiently large.  We want to see how large we need to take $n$ but before doing that we make
   the following remark.
   
   \begin{remark}
       \label{R:favorable}
      We know from Lemma \ref{L:pwl} that the piecewise linear function $S$ which interpolates the data  is in $K_w$ and therefore is itself a near optimal recovery with constant $C=2$.   In other  words, in this very special case, using an ad hoc analysis we can avoid solving
   a minimization problem and simply take the piecewise linear interpolant to the data.  Results of this special type are referred to as representer theorems and are preferred over minimization of loss functions when
   such representer theorems are known, see for e.g. \cite{unser2021unifying}.  The results of the present  article apply when representer theorems are not known.
   \end{remark}
   
   We proceed as if a representer theorem is not known to us  and instead we apply Theorem \ref{T3:discrete}. We want to see how  large we would need to take $n$ and how small we have to take $\mu$. For this, we need to give good estimates for $R(K_w)_{L_2[0,1]}$ and $R(K(w,2\e))_{L_2[0,1]}$, $\e>0$. We continue using the above notation, in particular for $h(\bx):=\max_{1\leq j<m}|x_{j+1}-x_j|$,  $S_w$  for the piecewise linear interpolant to the data $w$, and $I_j=[x_j,x_{j+1}]$, $1\leq j <m$. For notational convenience, we only consider the case  $1< p\le 2$.
   \begin{lemma}
       \label{L:boundR}
        If $1 < p\le 2$,
         $s=3/2-1/p$, and $R:=R(K_w)_{L_2[0,1]}>0$, then we have
     $\|S_w'\|_{L_p[0,1]} <1$ and  
       \be 
       \label{TR1}
       \Lambda h(\bx)^s\le R \le h(\bx)^s,\quad \Lambda:=1-\|S_w'\|_{L_p[0,1]},
       \ee 
       and 
       \be 
       \label{TR2}
        R(K(w,\e))_{L_2[0,1]}\le  h(\bx)^s+ \e\sqrt{m h(\bx)}.
       \ee 
   \end{lemma}
   \begin{proof}  We know from Lemma \ref{L:pwl} applied for the points $0=x_1<\ldots<x_m=1$ that $\|S_w'\|_{L_p[0,1]}\le 1$ and 
   $S_w\in K_w$.  
We first want to prove that if $R>0$, then we necessarily have $\Lambda>0$.  Indeed, if $R>0$, then there is a
   $g\in K_w$ such that $\|g-S_w\|_{L_2[0,1]}>0 $ and $\|g'\|_{L_p[0,1]} \leq 1$. The function $\frac{1}{2}(g+S_w)$ also is in $K_w$.
   Since $p>1$, then the strict convexity of the $L_p$ ball implies that 
   $$
   \frac{1}{2}\|g'+S_w'\|_{L_p[0,1]}< \frac 12 \left(\|g'\|_{L_p[0,1]} + \| S_w' \|_{L_p[0,1]} \right) \leq 1.
   $$   From  (i) of   Lemma \ref{L:pwl} (applied for $f=\frac{1}{2}(g+S_w)$) we derive that  $\|S_w'\|_{L_p[0,1]}<1$ and hence $\Lambda>0$.  

   Now,  consider \eref{TR1}.  The upper inequality follows from Lemma \ref{L:pwl} with $m=n$ and $\xi_j=x_j$, $1\leq j < m$. We next  prove the lower inequality.  Let $I=I_j$ be an interval $[x_j,x_{j+1}]$ with $|I_j|=h(\bx)$ and let $I_-$ be the left half of $I$ and $I_+$ be the right half of $I$.   We define  $g':= c[\chi_{I_+}-\chi_{I_-}]$ with $c:=\Lambda |I|^{-1/p}$, where $\chi_J$ denotes the characteristic function of an interval $J$.
The functions $S_w\pm g$ are both in $K_w$ and so $R\ge \|g\|_{L_2[0,1]}$.  Now $|g|$ is larger than $c|I|/4$ on the middle  half of $I$.  Therefore
$$
\|g\|_{L_2[0,1]}\ge c(|I|/4)(|I|/2)^{1/2}\ge \frac{1}{4\sqrt{2}}\Lambda |I|^{3/2-1/p}=\frac{1}{4\sqrt{2}} \Lambda h(\bx)^s.$$
This proves the lower inequality in \eref{TR1}.

Next, we prove \eref{TR2}.  Let $S_w$ and $S_{w'}$ be the piecewise linear interpolants for data $w$ and $w'$, respectively.  If $\|w-w'\|\le \e$, then
$$
\|S_w-S_{w'}\|^2_{L_2[0,1]}\le h(\bx) \sum_{j=1}^m|w_j-w'_j|^2 \le h(\bx) m\e^2.
$$
Now, if $f\in K(w,\e)$ and $\lambda(f)=w'$, then
\be 
\label{LbR}
\|f-S_w\|_{L_2[0,1]}\le \|f-S_{w'}\|_{L_2[0,1]}+\|S_w-S_{w'}\|_{L_2[0,1]}\le h(\bx)^s+ \e\sqrt{h(\bx)m},
\ee 
where we used Lemma~\ref{L:pwl}.
This proves \eref{TR2} and completes the proof of the lemma.
\end{proof}

We use the above lemma to see how large we have to take $n$ and how to choose $\mu$ in Theorem \ref{T3:discrete} in order to guarantee that
$\hat f$ is a near optimal recovery with constant $C$  in the $X=L_2[0,1]$ norm. Note that $C_X=1$ in this case.  We need  that $C$ is large enough and $\mu$ is small enough so that
\be 
\label{musatisfy}
\e+2R(K(w,2\e))_{L_2[0,1]}\le CR(K_w)_{L_2[0,1]},\quad \e:=\mu\max(\mu+1,C_Y).
\ee 
  The lemma gives  the bounds $R(K_w)_{L_2[0,1]}\ge \Lambda h(\bx)^s$ and $R(K(w,2\e))_{L_2[0,1]}\le  h(\bx)^s+ 2\e\sqrt{m h(\bx)}$.  We see that \eref{musatisfy} holds provided
\be
\label{provided}
\e\le  h(\bx)^{s-\frac 1 2} m^{-\frac{1}{2}}  \quad {\rm and}\quad C=7\Lambda^{-1},
 \ee 
since $h(\bx)\ge 1/m$. It is enough to take  
$\e\le  m^{-s}$ (since $s>1/2$) and similarly  $\mu\leq m^{-s}$.  Looking back at Theorem \ref{T3:discrete}, we need the approximation error  
 $\delta$  (as measured in $C(\Omega)$) to satisfy $\delta\le \mu^2$.  Since the approximation accuracy in $C(\Omega)$ is $O(n^{-1+1/p})$, 
   we need 
   $$ n\ge m^{\frac{2s}{1-1/p}} \quad \hbox{and}\quad 
   \mu\le m^{-s}
   $$
in the minimization problem \eref{loss11} to find a near optimal recovery of $f$.   

A similar analysis can be made when using the loss function \eref{loss113}.
   This example illustrates when given a compact set $K$, how one determines how well
   $\Sigma_n$ must approximate $K$ and how we must choose $\mu$
   so that solving the minimization problem gives a near optimal recovery from the given data.
   }

   \subsubsection{Numerical experiments}
   We next discuss the numerical implementation of the optimization with the loss \eref{loss113} for this special $K=U(W^1(L_p[0,1]))$.   
   We consider $\Sigma_n$ to be the space of continuous piecewise linear functions with breakpoints $\xi_j=j/n$, $0 \leq j \leq n$.
   We can parameterize $\Sigma_n$ using the hat function basis
   $H_j$, $j=0,\dots,n$,  where $H_j$ is the continuous piecewise linear function which takes the value one at $\xi_j$ and the value $0$
   at all other $\xi_i$, $i\neq j$.  Then each ${\nnew g}\in\Sigma_n$ can be written as
   \be
   \label{repS}
  {\nnew g(x)=g_\bc(x):=\sum_{j=0}^nc_jH_j(x),\quad x\in [0,1],}
   \ee   
   where $\bc=(c_1,\dots,c_n)$. Consider now the loss as a function of the parameters
   $\bc=(c_1,\dots,c_n)$:
   \be
   \label{parloss_c}
   {\nnew\cL^*(\bc):=\cL_\mu(g_\bc),\quad \bc\in\R^n.}
   \ee 
   The loss function $\cL^*$  is strictly convex whenever $1<p<\infty$ because it is the composition of a strictly convex function with an affine function.
   
    To numerically compute the minimum of $\cL_\mu$ over
    $\Sigma_n$ we minimize $\cL^*$ over $\R^n$ and use the argument $\bc^*$ attaining this minimum to define the
    minimizer $\hat f:=g_{\bc^*}$.  To compute $\bc^*$ we   use gradient descent with a sufficiently small step size and an initial guess.  Since the loss is nonnegative and  its gradient is locally Lipschitz except at $\bc = 0$, the algorithm converges (see \cite{Armijo}).

     As a numerical example, we take
     \be
     \label{ex}
     f(x)=\frac 1 4 x^{\frac 1 2}, \quad x\in [0,1].
     \ee  
     This function is in $W^1(L_p[0,1])$ for all $p<2$.  As a specific model class that contains $f$, we take
     \be
     \label{ex1}
      K=U(W^1(L_p[0,1])),\ p=3/2.
     \ee 
     This gives that $s=5/6$.
     While  we can implement the {\nnew algorithm} for any data observations,
     in order to get a spectrum of performance results, we
     take   random data samples consisting of $m=10,20,40,80,160,320$   observations. The  additional observations are chosen randomly while retaining the previous random observations.   Thus, we have a nested set of observations.
     
     The random draws turned out  to give  the  following values for $h$:
     
     $$h(\bx) = 0.13, 0.13, 0.108, 0.062, 0.048, 0.021.$$  
     For each of these values of $m$, we choose $n = 2m$ and $\mu = 0.1 m^{-s}$.   Note that this choice of $n$ is less than suggested by the theoretical estimates.
    
    Figure~\ref{f:linear_learning_asympt} 
    gives a graph of the true recovery rate and compares it with the bound  $h(\bx)^{\frac 56}$ which is our bound for the Chebyshev radius of $K_w$ and hence optimal recovery rate for these data observations. 
    We observe an asymptotic decay better than $h(\bx)^s$ (because we have taken only  one function in $K_w$ and not the supremum over all possible $f\in K_w$).
   
    \begin{figure}[ht!]
   \begin{center} 
   \scalebox{0.8}{\input{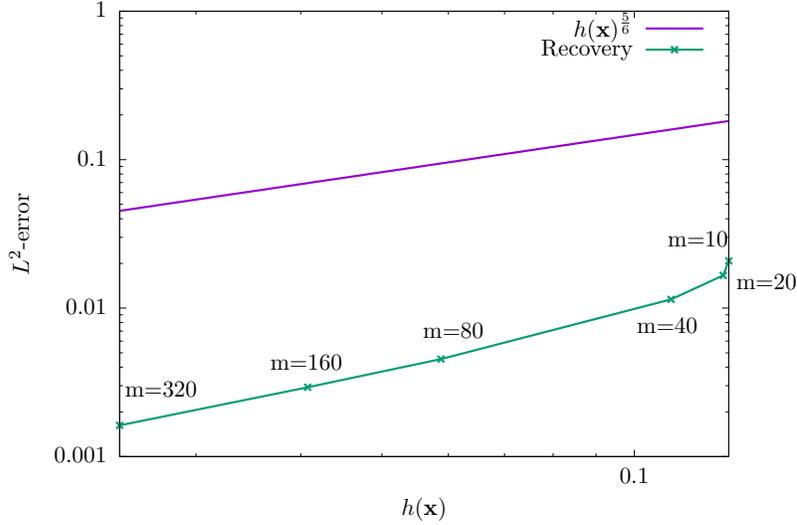}}
   \end{center}
   \caption{Recovery errors using $m=10, 20, 40, ..., 320$ random samples, $\mu = 0.1m^{-\frac 56}$, and $n=2m$. The error is compared with    $h(\bx)^{\frac 5 6}$ which is the   asymptotic behavior of the optimal error for the class $K=U(W^1(L_p(\Omega)))$, $p=\frac 32$.}\label{f:linear_learning_asympt}
   \end{figure}
    
    We next examine what happens if we do not use a penalty term, i.e., we take $\mu=0$
    and $n=2m$.  It is known that applying gradient descent
    with an initial choice of parameters converges to an interpolant which depends on the initial choice of parameters (see e.g. the discussion in \cite{DHP}).  We take the initial parameter choice as zero as we did in the case of a penalty term. 
   Figure~\ref{f:linear_learning_compare} compares the minimizing $\hat f$ for the case of $m=40$ without regularization ($\mu = 0$) 
  and with our proposed regularization ($\mu = 0.0046$). For the former, the over-parametrized {\nnew procedure}  produces a highly oscillating $\hat f$ that interpolates the data samples. In contrast, the constructed $\hat f$ for $\mu=0.0046$ exploits the regularity of $f \in W^1(L_p[0,1])$ and yields a recovery error $\| f - \hat f_\Sigma\|_{L_2[0,1]} = 0.011$, which is 10 times smaller than when using $\mu=0$. 
        
   \begin{figure}[ht!]
   \includegraphics[width=0.5\textwidth]{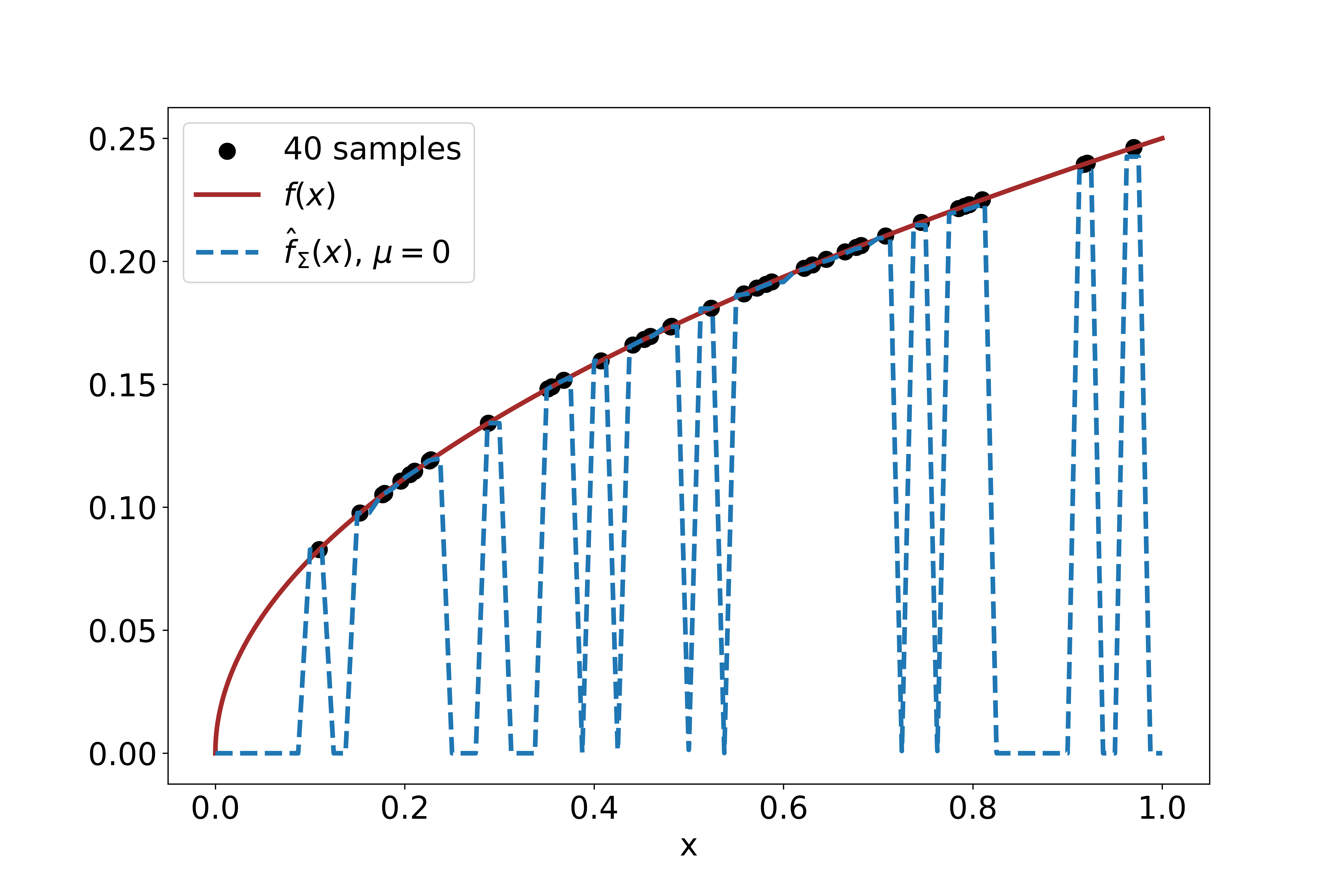}
   \includegraphics[width=0.5\textwidth]{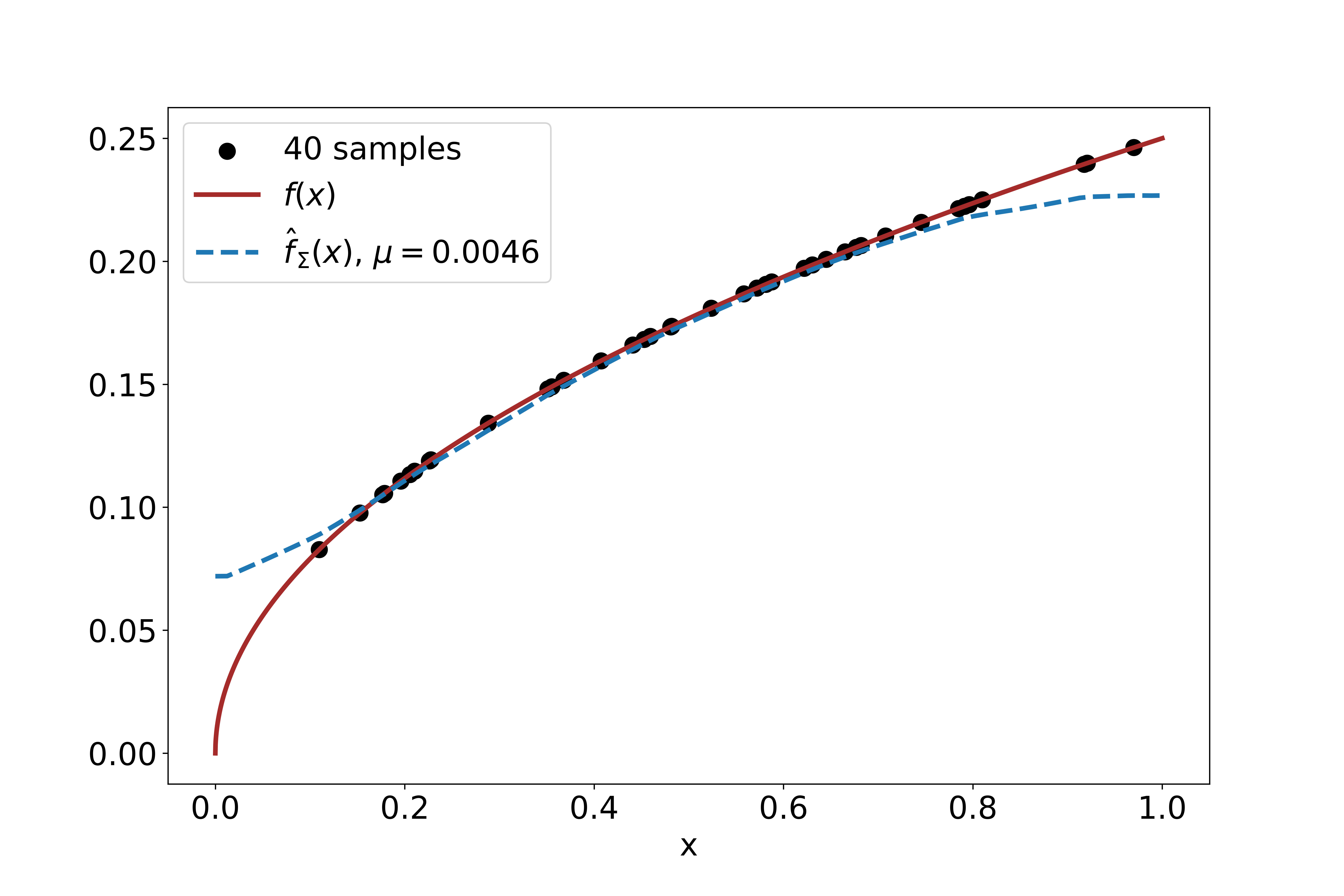}
   \caption{Learned function $\hat f_{\Sigma}$ in \eqref{loss113} with $p=3/2$, $\mu = 0$ (left) and $\mu = 0.0046$ (right), and using $40$ random samples (corresponding to $h(\bx) = 0.108$) of the function $f(x) = \frac 1 4 x^{\frac 1 2}$. The approximation space $\Sigma$ is the set of continuous piecewise linear functions subordinate to a uniform partition of $[0,1]$ using $n=80$ breakpoints. The recovery error $\| f - \hat f_\Sigma\|_{L_2[0,1]}$ with $\mu = 0.0046$ is $0.011$, which is 10 times smaller than when using $\mu=0$.}\label{f:linear_learning_compare}
   \end{figure}

   \subsection{Neural networks}
   \label{SS:Barron}
   It is now quite common in learning to take $\Sigma_n$
   as the space of outputs of a neural network depending on $n$ parameters.  We consider one often used example of this using ReLU activation.  Let $\Omega$ be the unit Euclidean ball in $\R^d$ with  $d\ge1$.  We consider the nonlinear space $\Sigma_n$  of  outputs of a single hidden layer ReLU  
   neural network of width $n$ on $\Omega$ (much less is known for deeper networks).   Each function ${\nnew g}\in\Sigma_n$ is of the form
   \be
   \label{ReLU} 
    {\nnew g(x)=c_0+\sum_{j=1}^n c_j (\omega_j\cdot x+b_j)_+,}
   \ee 
   where $\omega_j\in\R^d$, and the $c_j, b_j\in \R$.  This representation of ${\nnew g}$ is not unique.  We can  require $\omega_j$ to satisfy $\|(\omega_j)\|=\|(\omega_j)\|_{\ell_2}=1$ by adjusting the outer parameters $c_j$.  We can also require  the $b_j$ to be in $[-2,2]$. 
   The set $\Sigma_n$ is a nonlinear space of continuous piecewise linear functions on $\Omega$ .
   
   It is commonly thought that $\Sigma_n$ has significantly better  approximation properties than more traditional approximation methods based on polynomials,  splines, and wavelets and can therefore be more effective
   when learning a function $f$ from data.  This is especially thought to be true when $d$ is large as it is for many modern
   learning problems.  If this is indeed the case then it should be demonstrated through model classes $K$ whose elements
   can be better approximated by neural networks than by the traditional approximation methods.

   Accordingly, several model classes $K$ have been introduced and studied
   because they have favorable approximation properties when using $\Sigma_n$.  The most reknown  of these is the
   Barron class introduced in \cite{Barron} defined via Fourier
   transforms.  Several generalization of these classes (see e.g. \cite{EMW,SX,PN}) have been prominently studied.   Each of these model classes is of the form $K=U(Y)$ where $Y$ is a subspace of $C(\Omega)$.  They all have the feature that the functions in $K$ can be approximated
   in $L_q(\Omega)$, $1\le q\le \infty$, with an approximation rate $O(n^{-\alpha})$, $n\to\infty$, with $\alpha \ge 1/2$, and hence these model classes do not suffer the curse of dimensionality in terms of $d$.
   In going further with our discussion, we let $K$ be any of these model classes. We refer the reader to 
   \cite{DHP, SX} for  results on the approximation of functions in $K$ by the elements of  $\Sigma_n$.

   Optimal learning   for these classes can   be obtained via over-parameterized learning as described in Theorem \ref{T1:discrete}.  However, several important issues remain
   unresolved and prevent a complete theory for these model classes.  We describe these next where we assume the learning performance is to be measured in  $X=L_2(\Omega,\nu)$ metric.  Corresponding results are known for   $L_q$,  $1\le q\le \infty$, but in some cases are less precise. The discussion below should be compared
   with the previous subsection.
   
   Given data observations $w$ at data sites $\bx$ a major question that needs to be resolved is what
   is $R(K_w)_X$?  Results in this direction are for the most part unknown although some partial information can be obtained from Gelfand widths and sampling numbers.   Recall that
   Gelfand widths tell us the optimal learning rate that can be obtained for $K$ when using data given by $m$ linear functionals on $X$.  Upper bounds on Gelfand widths are given 
   in \cite{SX1} and one expects that these bounds are sharp.
   If we consider learning from point values of a function $f\in K$ the situation is more opaque. The sampling numbers for
   $K$ are not known.  Jonathan Siegel has provided us with an argument based on the Rademacher complexity of $K$ that shows that both the sampling numbers $s_m(K)_X$ and averaged sampling numbers $\bar s_m(K)_X$ of $K$ in $X$ are bounded
   by $Cm^{-1/4}$.  However, we do not know lower bounds for sampling numbers and what is perhaps more crucial is we do
   not know the near best positioning of the points $x_1,\dots,x_m$ at which to sample $f\in K$.  
 Some progress has been made recently in \cite{V}, where upper and lower bounds for sampling numbers for the smooth Barron classes have been obtained.
   Resolving these open questions is important in learning since it tells us how many samples we would need of a function $f\in K$ in order to recover it with a prescribed error $\e>0$.  Also, it would tell us how much over-parameterization we would need (how large to choose $n$ for $\Sigma_n$) to obtain optimal learning.

  When using  over-parameterized neural networks to solve the
  {\nnew finite-dimensional }  minimization in Theorem \ref{T1:discrete}, one can use   ridge regression or LASSO applied to the loss as a function of the coefficient in the representation \eref{ReLU}.  It is shown in \cite{PN} that there is always a minimizer which has a sparse representation \eref{ReLU}.  
  
  In summary, for these model
  classes $K$, we can numerically find a near optimal recovery of $f\in K$
  from given point data but we do not yet know the optimal learning rates nor do we know the optimal points where we should do the sampling.  Some crude bounds  bounds on performance and the amount of over-parameterization are known
  but  definitive results are still lacking.

 \section{Concluding remarks}
 \label{S:conclude}
 We have shown that optimal learning under a model class
 assumption $f\in K$ is always solved by an over-parameterized
 minimization problem.  The use of over-parameterization  matches what is typically done in modern machine learning.  However, it is important to point out that
 in many settings of modern learning one does not begin with a model class assumption
 and the loss function that is employed is simply  a least squares fitting of the data absent any penalty term.  In such a setting, i.e,  absent any model class assumption, there can be
 no theory to describe optimal performance since $f$ can be
 any function away from the data. 
 
 Another setting often studied is to employ neural networks $\Sigma_n$ in the loss function together with a regularization term in the loss function which
 penalizes the size of the parameters.  Such a penalty term can be viewed as imposing a model class assumption on the function to be learned.  A precise formulation of this connection must still be worked out.  One case where such a connection is known is when $K$ is the unit ball of the Radon BV space; see \cite{PN}.
 
 In the   setting without a model class assumption, as noted above, there are infinitely many solutions to the over-parameterized minimization problem.  The standard approach in learning is to choose one of these solutions by using a specific {\nnew procedure} to find an $\hat f$ corresponding to least squares loss.    The typical setting employs  over-parameterized deep neural networks in conjunction with minimization methods based on variants of gradient descent.  This is sometimes referred to as {\it deep learning}. The analysis of deep learning revolves around questions of whether such minimization procedures converge, how the limit
 depends on the initial parameter guess and the learning rate
 (step size in gradient descent), and if convergence does hold then what is the function $\hat f$ that is learned 
 (see the results on the Neural Tangent Kernel \cite{JGH,HN}).
  Another way to word  this approach is that one does not formulate a well defined learning problem (i.e. with a model class assumption)  but
 rather proposes a specific numerical method to utilize for learning and then centers the discussion on when this works well and why?   The current viewpoint is that the numerical method 
 implicitly imposes a model class assumption (described via neural tangent kernels).  Why such an implicit model class assumption
 is natural for the given learning setting is still to be
 explained.
 \vskip .1in
 \noindent
{\bf Acknowledgment:} The authors thank Professor Albert Cohen for helpful discussions 
on the research in this paper.  The authors also thank the referees
for excellent comments that improved this manuscript.
 The work of PB  was partially supported  by  the NSF Grants  DMS 1720297 and  DMS 2038080.
 The work of AB was partially supported  by  the NSF Grant
 DMS 2110811. The work of RD and GP was partially supported by the ONR Contract N00014-20-1-278, the    NSF Grant  DMS 2134077, and the NSF-Tripods Grant CCF-1934904.

 \vskip .1in
 \noindent
{\bf Declarations:}  

\noindent
Conflict of interest: The authors declare that they have no conflict of interest.

 \bibliographystyle{abbrv}


 \end{document}